\documentclass[10pt,conference]{IEEEtran}
\IEEEoverridecommandlockouts
\usepackage{cite}
\usepackage{amsmath,amssymb,amsfonts,amsthm}
\usepackage{algorithm}      
\usepackage{algorithmic}
\usepackage{graphicx}
\usepackage{textcomp}
\usepackage{xcolor}
\usepackage{colortbl}
\usepackage{wrapfig}
\usepackage{graphicx}
\usepackage{subcaption} 
\usepackage{booktabs}
\usepackage{multirow}
\usepackage{placeins}
\numberwithin{equation}{section}

\theoremstyle{plain}
\newtheorem{theorem}{Theorem}[section]
\newtheorem{proposition}[theorem]{Proposition}
\newtheorem{lemma}[theorem]{Lemma}
\newtheorem{corollary}[theorem]{Corollary}
\theoremstyle{definition}

\newtheorem{assumption}[theorem]{Assumption}
\theoremstyle{remark}

\newenvironment{itemize*}%
  {\begin{itemize}%
    \setlength{\itemsep}{0.8pt}%
    \setlength{\parskip}{0.8pt}}%
  {\end{itemize}}

\def\BibTeX{{\rm B\kern-.05em{\sc i\kern-.025em b}\kern-.08em
    T\kern-.1667em\lower.7ex\hbox{E}\kern-.125emX}}
\begin{document}

\title{Stabilizing Decentralized Federated Fine-Tuning via Topology-Aware Alternating LoRA}

\author{
Xiaoyu Wang, Xiaotian Li, Zhixiang Zhou, Chen Li, and Yong Liu\\
Department of Electrical and Computer Engineering, New York University, Brooklyn, NY, USA \\
Email: wang.xiaoyu@nyu.edu, xl3399@nyu.edu, zz4819@nyu.edu, chen.lee@nyu.edu, yongliu@nyu.edu
}

\maketitle

\begin{abstract}
Decentralized federated learning (DFL), a serverless variant of federated
learning, poses unique challenges for parameter-efficient fine-tuning due to
the factorized structure of low-rank adaptation (LoRA).
Unlike linear parameters, decentralized aggregation of LoRA updates introduces
topology-dependent cross terms that can destabilize training under dynamic
communication graphs.
We propose \texttt{TAD-LoRA}, a Topology-Aware Decentralized Low-Rank Adaptation framework that coordinates the updates and mixing of LoRA factors to control
inter-client misalignment. 
We theoretically prove the convergence of \texttt{TAD-LoRA} under
non-convex objectives, explicitly characterizing the trade-off between
topology-induced cross-term error and block-coordinate representation bias
governed by the switching interval of alternative training. Experiments under various communication conditions validate our analysis,
showing that \texttt{TAD-LoRA} achieves robust performance across different communication scenarios, remaining competitive in strongly connected topologies and delivering clear gains under moderately and weakly connected topologies, with particularly strong results on the MNLI dataset.
\end{abstract}


\section{Introduction}
Large Language Models (LLMs) such as GPT~\cite{achiam2023gpt},
LLaMA~\cite{touvron2023llama}, GLM~\cite{zengglm}, and DeepSeek~\cite{liu2024deepseek}
have achieved state-of-the-art performance on diverse language tasks. Adapting
these models in a privacy-preserving manner motivates federated learning (FL),
where data remain distributed across clients. 
Parameter-efficient fine-tuning (PEFT) methods, especially LoRA~\cite{hulora},
are attractive for FL due to their small memory and communication footprints.
LoRA introduces trainable low-rank factors $A$ and $B$, but aggregating these
matrices directly in FL creates a bilinear inconsistency: averaging $A$ and $B$
independently produces cross-client interactions $B_iA_j$ ($i\neq j$) that distort
the global update and degrade convergence. 

Several recent studies have attempted to mitigate this challenge. FLoRA~\cite{wangflora} eliminates cross terms via block-diagonal stacking, at the cost of parameter growth proportional to the number of clients. FlexLoRA~\cite{bai2024federated} reconstructs unified updates using SVD, but suffers from scalability issues due to costly server-side factorization. FFA-LoRA~\cite{sunimproving} fixes one LoRA matrix across all clients to prevent bilinear interference, though sacrificing the model expressiveness.

A recent line of work shows that \emph{alternating} the optimization of $A$ and
$B$ stabilizes LoRA aggregation in \emph{centralized} FL (CFL) by ensuring that
only one block is updated and aggregated per round~\cite{chen2025robust}.
However, this mechanism relies on strict synchronization: all clients must periodically synchronize parameters with the server and follow an identical alternative training schedule. In decentralized FL (DFL), where clients mix parameters only with their neighbors, information propagates gradually and client states drift within each phase~\cite{wang2025decentralized}. As a result, the key
assumptions enabling cross-term suppression in centralized alternating LoRA no longer hold in DFL setting. 

This paper revisits alternating LoRA through the lens of DFL and asks: {\it how can alternating low-rank updates be made stable and effective with asynchronous peer-to-peer model aggregation?} 

We show that decentralized mixing interacts with block-coordinate updates in a nontrivial way, leading to phase-state mismatch and drift between LoRA directions. To address these issues, we introduce \texttt{TAD-LoRA}, a topology-aware decentralized alternating LoRA 
framework that incorporates (i) interval-based directional switching,  and (ii) joint mixing of both LoRA blocks to maintain cross-client alignment. This yields a more stable alternating process and keeps decentralized LoRA training effective even on challenging NLP tasks.

Our main contributions are summarized as follows:
\begin{itemize*}
    \item We identify a fundamental instability of alternating LoRA in
    \emph{decentralized federated learning}, showing that asynchronous peer-to-peer mixing of factorized updates introduces a topology-dependent cross term and block-wise
    state mismatch that are absent in \emph{centralized federated learning}.

    \item We propose \texttt{TAD-LoRA}, a topology-aware decentralized alternating
    scheme that updates a single low-rank block at a time while jointly mixing
    both LoRA factors, enabling more consistent parameter states across clients
    under dynamic communication graphs.

    \item We provide the first convergence analysis of decentralized alternating
    LoRA, explicitly characterizing the trade-off between topology-induced
    cross-term error, which decreases with the switching interval $T$, and
    centralized block-coordinate representation bias, which increases with $T$.
    Under a local PL condition, we derive a topology-dependent optimal switching
    interval $T^\star$ that depends on the underlying communication reliability
    and network connectivity.

    \item Through extensive experiments across, we demonstrate that \texttt{TAD-LoRA} achieves consistent and
    robust performance across a wide range of communication scenarios: it remains competitive in strongly connected networks while providing clear gains under moderately and weakly connected networks, outperforming decentralized baselines on final
    accuracy, with particularly strong results on MNLI.
\end{itemize*}

\section{Related Work}

\subsection{Federated Learning}
Federated learning (FL) enables collaborative training across distributed
clients without sharing raw data~\cite{mcmahan2017communication}. Key
challenges include data heterogeneity~\cite{zhao2018federated}, partial client
participation~\cite{li2020federated}, and high communication
cost~\cite{kairouz2021advances}. To reduce communication overhead, prior work
explores model compression~\cite{sattler2019robust}, sparse
updates~\cite{aji2017sparse}, personalized models~\cite{fallah2020personalized},
and adaptive aggregation schemes. Decentralized FL~\cite{lian2017can, lalitha2019peer, wang2025decentralized} and
asynchronous FL~\cite{xie2019asynchronous} offload or eliminate the central server,
making client states evolve at different rates and complicating convergence
analysis. Our work examines these issues specifically for LoRA-based low-rank
adaptation, where the bilinear parameterization introduces unique aggregation
inconsistencies under both CFL and DFL.

\subsection{LoRA in Federated Learning}
LoRA~\cite{hulora} is widely adopted in FL due to its low communication cost and
parameter efficiency. Several variants extend LoRA to heterogeneous and
personalized settings, including SLoRA~\cite{babakniya2023slora},
FedSA-LoRA~\cite{guo2025selective}, and
pFedLoRA~\cite{yi2024pfedloramodelheterogeneouspersonalizedfederated}. These
methods address client diversity but do not resolve the bilinear inconsistency
inherent to aggregating independently trained LoRA factors. A complementary line of work tackles this issue directly. FLoRA~\cite{wangflora}
uses block-diagonal concatenation to avoid cross terms but enlarges model size;
FlexLoRA~\cite{bai2024federated} reconstructs global updates via SVD; and
FFA-LoRA~\cite{sunimproving} freezes one LoRA matrix to suppress cross-client
interference. While effective, these methods trade off scalability, flexibility,
or representational capacity.

Most relevant to this work, RoLoRA~\cite{chen2025robust} shows that alternating
updates of $A$ and $B$ improve stability in \emph{centralized} FL. However, its
analysis assumes synchronous rounds and server-enforced parameter consistency,
and evaluates only fixed odd--even schedules under centralized aggregation.
Whether alternating LoRA remains stable in decentralized peer-to-peer settings—
where clients mix stale parameters and cannot maintain a shared frozen block—
has not been investigated, and naive extensions behave poorly in practice.

Beyond federated settings, AltLoRA~\cite{yu2025altlora} applies alternating
projections to improve gradient approximation in centralized single-machine
fine-tuning. This literature focuses on optimization quality rather than
multi-client aggregation and is orthogonal to our study of alternating LoRA
under communication-constrained FL.

\section{Motivation and Analysis}

LoRA models a weight update using a low-rank factorization $\Delta W = BA$.
In federated learning, each client updates  local \emph{factors} $(A_i,B_i)$, and aggregation is performed on each factor   rather than their product:  
\[
A^{\mathrm{agg}}=\sum_i w_iA_i,\qquad
B^{\mathrm{agg}}=\sum_i w_iB_i,
\]
and reconstructing the update gives
\[
B^{\mathrm{agg}}A^{\mathrm{agg}}
= \underbrace{\sum_i w_i^2 B_iA_i}_{\text{desired average update}}
\;+\;
\underbrace{\sum_{i\neq j} w_iw_j\, B_iA_j}_{\text{cross-client interference}}.
\]
The first term corresponds to the desired average update, while the second mixes
$B_i$ from one client with $A_j$ from another---an update direction that no
client ever computed. These cross terms are a primary source of instability when LoRA is trained on heterogeneous data.

\paragraph{Why alternating works in centralized FL.}
In centralized federated learning (CFL), alternating updates eliminate these
cross terms~\cite{chen2025robust}. During a B-phase, all clients share the same $A$, so $B_iA$ averages cleanly; during
an A-phase, they share the same $B$, so $BA_i$ averages cleanly.
Server-enforced synchronization ensures that the frozen block is identical across
clients, restoring the stability of two-block coordinate descent.

\paragraph{Challenges of Alternating LoRA in Decentralized Settings.}

In decentralized federated learning (DFL), this synchronization assumption no
longer holds.
Clients exchange parameters only with neighbors, so even within the same logical
phase, their local models may differ due to incomplete information propagation.
As a result, the frozen block is no longer shared across clients, and alternating
updates can no longer fully suppress cross terms.
This leads to two DFL-specific effects:
\begin{itemize*}
    \item \textbf{Phase-state mismatch:} clients enter the same phase with
    different versions of the frozen block.
    \item \textbf{Block-wise drift:} discrepancies in the frozen block accumulate
    over time, reintroducing cross-client interactions during mixing.
\end{itemize*}
The issue can be illustrated by the following example:
\[
\begin{aligned}
&\text{client $i$ holds } A_i \text{ and receives } A_j \neq A_i \text{ from neighbors} \\
&\Longrightarrow
\text{the B-phase can no longer eliminate } B_i A_j \text{ cross terms}.
\end{aligned}
\]

These effects induce a fundamental trade-off.
Methods such as \texttt{RoLoRA}, which alternate frequently without coordinating
frozen-block mixing, suffer from block misalignment under sparse communication.
In contrast, \texttt{FFA-LoRA}, which uses very infrequent switching, mitigates
misalignment but amplifies cross-term noise due to prolonged staleness.
This trade-off suggests that the switching interval must be jointly considered
with communication conditions, motivating the topology-aware coordination
mechanism in \texttt{TAD-LoRA}.

\section{Methodology}
\begin{figure}[t]
  \centering
  \includegraphics[width=\linewidth]{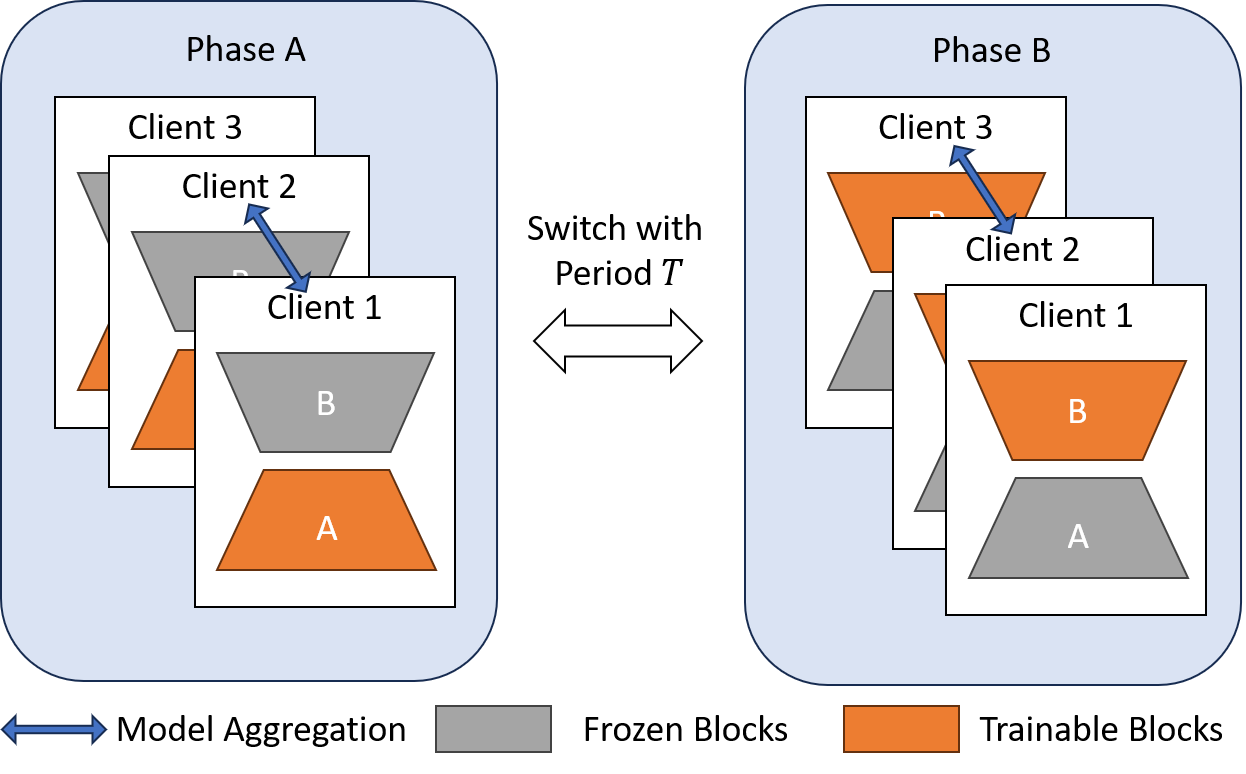}
  \caption{
  Overall illustration of \texttt{TAD-LoRA} under decentralized federated learning (DFL).
  Clients communicate directly with a dynamically selected subset of peers
in a peer-to-peer manner, as determined by the underlying communication topology.
  At each alternating phase (lasting $T$ rounds), A-blocks or B-blocks is actively
updated and coordinated across clients, while the other remains frozen during
that phase and is aligned implicitly through topology-aware switching over time.}
  \label{fig:tad_lora_overall}
\end{figure}

The above analysis shows that decentralized communication disrupts the
synchronization required for alternating LoRA to remain stable, mainly due to
phase-state mismatch and block-wise drift resulted from asynchronous P2P model mixing.  To study and mitigate these effects,
we first formalize the decentralized FL communication model, and then introduce
our joint-mixing alternating LoRA procedure. Overview of the \texttt{TAD-LoRA} framework is at Fig~\ref{fig:tad_lora_overall}
\subsection{Decentralized FL Communication Model}
We adopt a standard decentralized FL (DFL) setting where clients maintain local
parameters and exchange them with neighbors through a time-varying mixing
matrix $W_t$:
\[
x_i^{t+1}=\sum_{j=1}^N (W_t)_{ij}\, x_j^t.
\]
The matrices $\{W_t\}$ are \emph{doubly-stochastic} and satisfy standard connectivity and
spectral-gap assumptions used in decentralized optimization. All clients perform local updates every round. Communication is serverless:
each round consists of (i) local computation and (ii) optional peer-to-peer
mixing when communication is triggered. This contrasts with centralized FL
(CFL), where a parameter server enforces identical global states at each round.

\subsection{Alternating LoRA in FL}

LoRA represents a low-rank update as $\Delta W = BA$, where $A$ and $B$
denote the down- and up-projection matrices. Alternating LoRA updates only
one block per phase:
\[
A^{t+1}_i = A^t_i - \eta\nabla_A\mathcal{L}_i(A^t_i,B^t_i),\qquad  
B^{t+1}_i = B^t_i,
\]
during the A-phase, and symmetrically for the B-phase. This block splitting 
removes LoRA cross terms and stabilizes optimization.

\noindent{\bf CFL: Centralized Alternating LoRA (RoLoRA).}
In centralized FL, all clients follow identical A/B phases and the server conducts model aggregation through FedAvg:
\[
A^{t+1} = \frac1N\sum_{i=1}^m A^{t+1}_i,\qquad
B^{t+1} = \frac1N\sum_{i=1}^m B^{t+1}_i.
\]
Since only the active block is aggregated in each phase, the centralized
procedure behaves as standard two-block coordinate descent.

\noindent{\bf DFL: Decentralized Alternating LoRA Baseline.}
In DFL, aggregation is replaced by peer-to-peer mixing, but only the \emph{active}
block is synchronized. During an A-phase:
\[
A^{t+1}_i = \sum_j (W_t)_{ij} A^t_j, \qquad
B^{t+1}_i = B^t_i,
\]
and during a B-phase:
\[
B^{t+1}_i = \sum_j (W_t)_{ij} B^t_j, \qquad
A^{t+1}_i = A^t_i.
\]
Because the frozen block is not synchronized across clients during its inactive
phase, clients gradually accumulate
inconsistent versions of it, leading to \emph{phase-state mismatch} and
block-wise drift, the main failure mode of alternating LoRA under DFL.

\subsection{TAD-LoRA: Joint-Mixing Alternating LoRA}
To mitigate drift in decentralized alternating LoRA, we introduce
\texttt{TAD-LoRA}, which synchronizes both LoRA blocks in every round,
regardless of which block is being updated. After local updates, each client
performs joint mixing:
\[
A^{t+1}_i = \sum_j (W_t)_{ij} A^t_j,\qquad
B^{t+1}_i = \sum_j (W_t)_{ij} B^t_j.
\]
Joint mixing reduces the frozen block  misalignment across the network, mitigating 
phase-state mismatch and stabilizing alternating LoRA under decentralized
communication.

We study a decentralized learning process over $R$ communication rounds, where
each round $t\in\{0,\ldots,R-1\}$ consists of a local update followed by a
communication-induced mixing step.

Algorithm~\ref{alg:TAD_LoRA} summarizes the full procedure.

\begin{algorithm}[t]
\caption{TAD-LoRA: Joint-Mixing Alternating LoRA in Decentralized FL}
\label{alg:TAD_LoRA}
\begin{algorithmic}[1]
\REQUIRE Initial blocks $\{(A_i^{0}, B_i^{0})\}_{i=1}^m$; total rounds $R$; switching interval $T$; mixing matrices $\{W_t\}_{t=0}^{R-1}$.
\FOR{$t = 0,\ldots,R-1$}
  \IF{$\left\lfloor t/T \right\rfloor$ is even} 
    \STATE \textbf{B-phase:} each client $i$ updates $B_i^{t}$ locally while keeping $A_i^{t}$ fixed.
  \ELSE
    \STATE \textbf{A-phase:} each client $i$ updates $A_i^{t}$ locally while keeping $B_i^{t}$ fixed.
  \ENDIF
  \STATE \textbf{Joint mixing:} for each client $i$,
  \STATE \hspace{1em}$A_i^{t+1} \gets \sum_{j=1}^m (W_t)_{ij}\,A_j^{t}$
  \STATE \hspace{1em}$B_i^{t+1} \gets \sum_{j=1}^m (W_t)_{ij}\,B_j^{t}$
\ENDFOR
\end{algorithmic}
\end{algorithm}

Throughout the paper, we use the terms \emph{iteration} and \emph{round}
interchangeably to denote one local update followed by a communication-induced
mixing step.
\section{Convergence  Analysis \& Switching Interval}\label{theoretical_analysis}

We present a convergence analysis of decentralized alternating LoRA under time-varying topologies, focusing on the trade-off governed by the switching interval $T$. We demonstrate that while a larger $T$ enhances inter-client consensus by suppressing topology-induced errors, it inevitably introduces a \emph{centralized representation bias} due to stale updates. To quantify this balance, we decompose the averaged LoRA update and establish a stationarity bound under dynamic connectivity. Finally, assuming a local PL condition, we derive the convergence rate and characterize the topology-dependent optimal switching interval $T^\star(\rho)$.

\subsection{Setup and Assumptions}
We consider the LoRA parameterization $\theta = \theta_0 + BA$, where client $i$ holds local blocks $(A_i^t, B_i^t)$ to optimize the global objective $F(\theta) = \frac{1}{m}\sum_{i=1}^m f_i(\theta)$. Communication is governed by time-varying, doubly-stochastic mixing matrices $W_t$ satisfying the mean-square contraction property $\mathbb{E}\|W_t - \frac{1}{m}\mathbf{1}\mathbf{1}^\top\|_2^2 \le \rho^2$ with $\rho \in (0, 1)$.

For analysis, we define the averaged blocks $\bar A^t, \bar B^t$ and the averaged model $\bar\theta^t$. Consensus is measured by the block disagreements $\|\Delta_A^t\|^2 := \frac{1}{m}\sum_{i}\|A_i^t-\bar A^t\|_F^2$ and $\|\Delta_B^t\|^2 := \frac{1}{m}\sum_{i}\|B_i^t-\bar B^t\|_F^2$.

In addition to standard regularity conditions (e.g., $L$-smoothness, bounded stochastic gradients) detailed in Appendix~\ref{app:assumptions}, our analysis relies on two structural assumptions specific to the alternating LoRA landscape.
To streamline notation, we define the following reference values:
\begin{itemize}
\item Let $\theta_1^\star$ be the optimal parameter of the centralized alternating LoRA with $T=1$.
\item Let $F_1^\star := F(\theta_1^\star)$ denote the corresponding optimal function value.
\item Let $F_T^\star$ denote the optimal function value achievable with a switching interval $T$.
\end{itemize}

\begin{assumption}[Local Polyak--{\L}ojasiewicz (PL) condition on the LoRA subspace]\label{ass:pl-main}
There exists a constant $\mu>0$ such that for all $\theta$ in a neighborhood of $\theta_1^\star$, the suboptimality relative to the $T=1$ baseline is bounded by the gradient:
\[F(\theta) - F^\star_1 \le \frac{1}{2\mu}\|\nabla F(\theta)\|^2.\]
\end{assumption}

While the PL condition governs the convergence rate, coarse-grained alternation ($T > 1$) introduces a structural bias by restricting the optimization trajectory. We quantify this cost relative to the reference $T=1$ baseline as follows.
\begin{assumption}[Alternating-induced factorization bias]\label{ass:alt-bias}
There exists a constant $C_3 > 0$ such that, for a sufficiently small stepsize $\eta$, the gap between the $T$-interval optimum and the $T=1$ baseline satisfies:$$  \phi(T)
  :=
  F^\star_{T} - F^\star_1
  \le
  C_3\,\eta^2T.$$
\end{assumption}

\subsection{Error Decomposition and Consensus Mechanism}

Unlike linear models, the distributed averaging of LoRA parameters introduces a non-linear discrepancy. To analyze this, we decompose the global update matrix $\bar W^t := \frac{1}{m}\sum_{i} B_i^t A_i^t$ into a centralized-equivalent term and a topology-dependent error:$$  \bar W^t = \bar B^t\bar A^t + C^t, \quad \text{where} \quad \|C^t\|_F \le \|\Delta_A^t\|\,\|\Delta_B^t\|.$$

This inequality (derived via Cauchy--Schwarz) reveals that the deviation from the centralized trajectory is controlled strictly by the consensus error. This motivates our analysis of how the switching interval $T$ affects block disagreement.

Under standard assumptions (smoothness, bounded gradients, and mixing contraction $\rho$), we establish that frozen blocks contract purely via gossip, while updated blocks reach a steady-state error (Lemma~\ref{lem:block-consensus-appendix} in Appendix). Consequently, increasing the switching interval allows the frozen block to benefit from $T-1$ additional gossip steps. We quantify this effect by bounding the cycle-averaged cross term 
(see Appendix~\ref{sec:cross-term-decay}):$$  \frac1T\sum_{\tau=0}^{T-1}\mathbb E\|C^{t+\tau}\|_F
\;\le\;
\frac{C_{\mathrm{cr}}\eta^2}{T(1-\rho)}.$$

This result is pivotal: it demonstrates that extending the alternating period $T$ explicitly suppresses the topology-induced error by a factor of $1/T$, effectively "buying" consensus with time. 

\subsection{Convergence and Topology-Dependent Trade-off}

Building on the cross-term bound, we derive the global convergence guarantees. We first establish a stationarity bound for the averaged model $\bar\theta^t$, which allows us to translate the gradient norm into function-value suboptimality. By incorporating the Local PL condition (Assumption~\ref{ass:pl-main}) and the alternating-induced bias $\phi(T)$ (Assumption~\ref{ass:alt-bias}), we obtain our main convergence result relative to the centralized baseline $F_1^\star$:
\begin{theorem}[Convergence Rate and Trade-off]\label{thm:subopt-main}
Let $\hat\theta_R$ be sampled uniformly from the trajectory. The suboptimality is bounded by:
\begin{equation}
\begin{aligned}
\mathbb E\big[F(\hat\theta_R)-F_1^\star\big]
\le\;
&\underbrace{\frac{1}{2\mu}\left( \frac{C_0}{\eta R} + C_1\eta \right)}_{\text{Optimization Error}} \\
&+ \underbrace{\frac{C_2\eta^2}{2\mu T(1-\rho)}}_{\text{Topology Error}}
+ \underbrace{\phi(T)}_{\text{Bias}}.
\end{aligned}
\end{equation}
\end{theorem}

Theorem~\ref{thm:subopt-main} formalizes the fundamental tension in our method. Neglecting vanishing optimization terms, the performance is governed by $\Psi(T) := \frac{C_2\eta^2}{2 \mu T(1-\rho)} + \phi(T)$. The topology error decreases with $T$, while the representation bias $\phi(T)$ (bounded by $O(T\eta^2)$) increases with $T$. Minimizing $\Psi(T)$ yields the topology-dependent optimal interval:$$  T^\star(\rho) \simeq \Theta\left(\frac{1}{\sqrt{1-\rho}}\right).$$This implies that in poorly connected networks (large $\rho$), a larger switching interval $T$ is necessary to suppress communication errors, whereas well-connected networks favor smaller $T$ to mitigate representation bias.

\noindent{\bf Discussion.}
In practice, due to discrete scheduling and training noise, this tradeoff often manifests as a range of effective switching intervals
rather than a single sharp optimum. Our analysis establishes a non-monotonic dependence of performance on the switching interval, implying the existence of an optimal regime rather than a universally best fixed choice. We will empirically examine this behavior in the following section. 

\section{Experiments}

\subsection{Experimental Setup}
We evaluate \texttt{TAD-LoRA} on four representative GLUE tasks (SST-2, QNLI, QQP, MNLI). Unless otherwise specified, we follow the experimental setup of~\cite{sunimproving} for all tasks, models, and training configurations.

\subsubsection{Model Configuration.}

We use \textbf{RoBERTa-Large} (335M) with LoRA applied to the \textbf{Q}/\textbf{V} projections ($r=8$, $\alpha=16$, dropout~0.1).  The classification head is frozen.

\subsubsection{Federated Learning Settings.}
\paragraph*{Data Partitions.}

For binary tasks, 10 clients follow $3\times[0.9,0.1]$, $3\times[0.1,0.9]$, $4\times[0.5,0.5]$;  
for MNLI, $4\times[0.9,0.05,0.05]$, $3\times[0.05,0.9,0.05]$, $3\times[0.05,0.05,0.9]$.

\paragraph*{Communication Topology.}
Different from~\cite{sunimproving}, we consider a decentralized federated learning (DFL) setting with explicit network topologies. We adopt a commonly used random communication topology, corresponding to an
Erd\H{o}s--R\'enyi graph~\cite{erdds1959random} with edge activation probability $p$, for decentralized training. At each round, each client exchanges model updates with peers independently with
probability $p$, capturing typical information mixing behavior.
We vary $p \in \{0.5, 0.2, 0.1, 0.05, 0.02, 0.01\}$, covering regimes from dense to extremely sparse communication. Due to space constraints, the main results focus on this representative topology,
while results under more structured topologies (e.g., ring networks) are reported in the appendix.

\paragraph*{Choice of switching intervals.}
We fix the total training horizon to $R=150$ rounds.
To avoid asymmetric updates caused by an incomplete final alternation cycle,
we only consider switching intervals $T$ that evenly divide $R$.
In preliminary experiments, we observed higher run-to-run variance when $T \nmid R$
(e.g., $T=7$ or $T=20$).
We therefore restrict $T$ to divisors of $R$ (e.g., $T\in\{1,2,3,5,10,15\}$) for stable and reproducible estimation of $\hat T^*(p)$.

\subsubsection{Baselines.}
We compare \texttt{TAD-LoRA} with following baselines:
\begin{itemize*}
    \item \texttt{LoRA}: Standard LoRA fine-tuning under decentralized federated learning, where both LoRA matrices are locally trained and aggregated via FedAvg.
    \item \texttt{FFA-LoRA}~\cite{sunimproving}: A LoRA variant that freezes $A$ and only updates $B$, effectively restricting adaptation to a fixed low-dimensional subspace.
    \item \texttt{RoLoRA}~\cite{chen2025robust}: A centralized alternating LoRA method extended to DFL by aggregating only the active trainable parameters. Following the original paper, we use a per-round alternation, i.e., $T=1$.
    \item \texttt{TAD-LoRA}: our topology-aware joint-mixing decentralized alternating LoRA method.
\end{itemize*}

\subsubsection{Training and Evaluation Details.}
Each FL round uses \textbf{20 local steps}, for \textbf{150 rounds} total.  
We use AdamW (HuggingFace defaults), sequence length 128, batch size 32, and search learning rate over  
$\{2\times10^{-4} , 5\times10^{-4}, 1\times10^{-3}, 2\times10^{-3}, 5\times10^{-3}\}$.  
All experiments run on NVIDIA RTX8000 clusters. In DFL, no global model exists. For each seed, we evaluate all 10 client models and compute the \textbf{mean accuracy across clients}.  
We then report the \textbf{mean and standard deviation across random seeds} as the final performance.

\subsection{Main Results}
\noindent{\bf Overall performance under varying communication probabilities.}
\begin{figure}[t]
  \centering
  \includegraphics[width=0.9\linewidth]{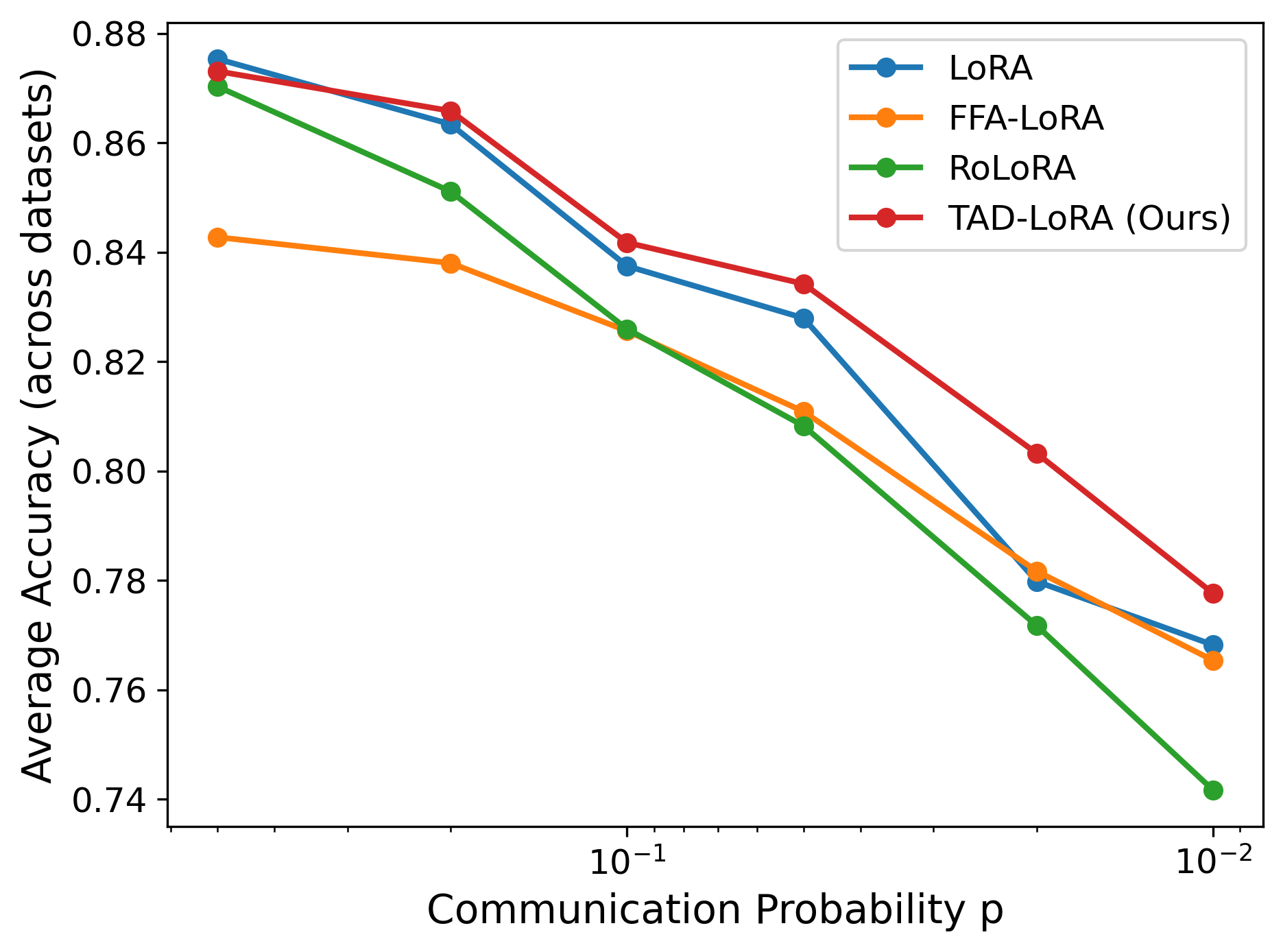}
  \caption{
  Average test accuracy across datasets under different communication probabilities $p$.
  \texttt{LoRA} denotes the vanilla decentralized LoRA baseline with FedAvg\cite{mcmahan2017communication};
  \texttt{FFA-LoRA} fixes one LoRA factor during training;
  \texttt{RoLoRA} alternates LoRA factors in a naive round-robin manner;
  \texttt{TAD-LoRA} is our topology-aware decentralized alternating LoRA.
  As communication becomes weaker (smaller $p$),
  \texttt{TAD-LoRA} consistently outperforms all baselines,
  while remaining competitive under strong communication.
  }
  \label{fig:avg_acc_vs_p}
\end{figure}
Figure~\ref{fig:avg_acc_vs_p} shows the average test accuracy across datasets
under different communication probabilities $p$.
While all methods (except \texttt{FFA-LoRA}) perform similarly under strong
communication, clear performance gaps emerge as communication becomes sparser.
In particular, \texttt{TAD-LoRA} exhibits increasingly larger gains as $p$
decreases, consistently outperforming prior methods in both moderate and weak
communication regimes.
These results indicate that naive alternating strategies fail to properly handle
staleness and block misalignment under limited communication, whereas
topology-aware alternating becomes increasingly important as communication
constraints intensify.
We also observe complementary trends between \texttt{RoLoRA} and \texttt{FFA-LoRA},
which provide additional evidence for our analysis.
\texttt{RoLoRA} can be viewed as an alternating baseline without coordinating the
mixing of the frozen block, whereas \texttt{FFA-LoRA} corresponds to an extreme
case with very infrequent switching (i.e., large $T$).
Consistent with the proposed trade-off, \texttt{RoLoRA} tends to perform well
when communication is strong but degrades rapidly as $p$ decreases, suggesting
that insufficient frozen-block mixing leads to pronounced misalignment under
sparse communication.
In contrast, \texttt{FFA-LoRA} becomes relatively more competitive in low-$p$
regimes, aligning with the intuition that less frequent alternation can mitigate
cross-term effects when synchronization is severely limited.
Overall, these contrasting behaviors corroborate the need to jointly coordinate
switching and block mixing, which is exactly what \texttt{TAD-LoRA} achieves.
Both baselines are limited by incomplete \emph{client-level coordination} of the
active LoRA block across alternating phases, which becomes increasingly harmful
as communication becomes sparse.

\begin{table*}[t]
\centering
\caption{Test accuracy (mean $\pm$ variance) under representative strong ($p=0.5$), moderate ($p=0.1$),
and extremely weak ($p=0.02$) communication regimes.
For each $p$, best and second-best results are highlighted.}
\label{tab:main_results_combined}
\begin{tabular}{c|l|cccc|c}
\toprule
$p$ & Method & SST-2 & QQP & QNLI & MNLI & Avg. \\
\midrule
\multirow{4}{*}{0.5}
 & \texttt{LoRA}
 & \textbf{0.9468} $\pm$ 0.0019
 & \textbf{0.8347} $\pm$ 0.0085
 & \textbf{0.9067} $\pm$ 0.0043
 & \underline{0.8132} $\pm$ 0.0194
 & \textbf{0.8754} \\
 & \texttt{FFA-LoRA}
 & 0.9436 $\pm$ 0.0032
 & 0.8051 $\pm$ 0.0061
 & 0.8911 $\pm$ 0.0036
 & 0.7313 $\pm$ 0.0333
 & 0.8428 \\
 & \texttt{RoLoRA}
 & \underline{0.9462} $\pm$ 0.0023
 & 0.8216 $\pm$ 0.0108
 & \underline{0.9021} $\pm$ 0.0031
 & 0.8115 $\pm$ 0.0088
 & 0.8703 \\
 & \texttt{TAD-LoRA} (Ours)
 & 0.9448 $\pm$ 0.0007
 & \underline{0.8328} $\pm$ 0.0072
 & 0.9003 $\pm$ 0.0048
 & \textbf{0.8145} $\pm$ 0.0066
 & \underline{0.8731} \\
 \midrule
 \multirow{4}{*}{0.1}
 & \texttt{LoRA}
 & \underline{0.9370} $\pm$ 0.0058
 & \textbf{0.8098} $\pm$ 0.0076
 & \underline{0.8779} $\pm$ 0.0159
 & \underline{0.7253} $\pm$ 0.0440
 & \underline{0.8375} \\
 & \texttt{FFA-LoRA}
 & 0.9313 $\pm$ 0.0046
 & 0.7915 $\pm$ 0.0086
 & 0.8710 $\pm$ 0.0048
 & 0.7086 $\pm$ 0.0242
 & 0.8256 \\
 & \texttt{RoLoRA}
 & 0.9325 $\pm$ 0.0048
 & 0.7890 $\pm$ 0.0027
 & 0.8711 $\pm$ 0.0045
 & 0.7113 $\pm$ 0.0133
 & 0.8260 \\
 & \texttt{TAD-LoRA} (Ours)
 & \textbf{0.9401} $\pm$ 0.0019
 & \underline{0.8050} $\pm$ 0.0073
 & \textbf{0.8815} $\pm$ 0.0145
 & \textbf{0.7405} $\pm$ 0.0376
 & \textbf{0.8418} \\
\midrule
\multirow{4}{*}{0.02}
 & \texttt{LoRA}
 & 0.8668 $\pm$ 0.0811
 & \underline{0.7702} $\pm$ 0.0066
 & \underline{0.8416} $\pm$ 0.0031
 & \underline{0.6407} $\pm$ 0.0118
 & 0.7798 \\
 & \texttt{FFA-LoRA}
 & \underline{0.9186} $\pm$ 0.0038
 & 0.7627 $\pm$ 0.0121
 & 0.8264 $\pm$ 0.0178
 & 0.6191 $\pm$ 0.0358
 & \underline{0.7817} \\
 & \texttt{RoLoRA}
 & 0.9147 $\pm$ 0.0022
 & 0.7582 $\pm$ 0.0116
 & 0.8160 $\pm$ 0.0045
 & 0.5980 $\pm$ 0.0291
 & 0.7717 \\

 & \texttt{TAD-LoRA} (Ours)
 & \textbf{0.9263} $\pm$ 0.0019
 & \textbf{0.7783} $\pm$ 0.0036
 & \textbf{0.8480} $\pm$ 0.0063
 & \textbf{0.6604} $\pm$ 0.0408
 & \textbf{0.8032} \\
\bottomrule
\end{tabular}
\end{table*}

\noindent{\bf Quantitative comparison under representative regimes.}
Table~\ref{tab:main_results_combined} reports detailed results under
representative strong ($p=0.5$), moderate ($p=0.1$), and weak ($p=0.02$)
communication regimes.
Under strong communication ($p=0.5$), \texttt{TAD-LoRA} achieves comparable performance to the strongest baselines.
Under moderate communication ($p=0.1$), \texttt{TAD-LoRA} already achieves the
best or near-best performance across individual datasets, demonstrating that
the benefits of topology-aware alternating emerge well before entering extreme
communication sparsity.
Under weak communication ($p=0.02$), \texttt{TAD-LoRA} consistently attains the
strongest performance across nearly all datasets, rather than being driven by
improvements on a single task.
Together, these results highlight the per-dataset robustness of
\texttt{TAD-LoRA} and its effectiveness across a wide spectrum of communication
conditions.

\subsection{Ablation Study}
\noindent{\bf Topology-aware selection of the switching interval.}
Our theoretical analysis suggests that the optimal switching interval should
adapt to the underlying network connectivity.
For Erd\H{o}s--R\'enyi graphs with edge activation probability $p$,
the effective connectivity improves monotonically with $p$,
leading to different preferred switching regimes across communication conditions.

Motivated by this insight, we evaluate \texttt{TAD-LoRA} under
different values of $p$ and sweep the switching interval $T$ accordingly.
We observe that sparser graphs (smaller $p$) tend to favor larger switching
intervals, while denser graphs achieve their best performance with more frequent
alternation.
This behavior is consistent with the non-monotonic dependence on $T$
predicted by the theory (see Appendix~\ref{appendix:ER}).

\noindent{\bf Aggregate trend of the optimal switching interval}
While the theory characterizes a population-level trend,
we empirically estimate $\hat T^*(p)$ for analysis purposes by selecting the
best interval from a \emph{discrete candidate set} considered in the experimental
setup, via a \emph{noisy argmax} over performance averaged across multiple random
seeds and a finite training horizon.
As a result, $\hat T^*(p)$ is sensitive to discretization effects and run-to-run
variance, which may introduce local deviations from the overall trend.

To mitigate such noise, Figure~\ref{fig:bestT_median} reports the
\emph{median optimal switching interval across datasets}.
Focusing on the reliably convergent regime $p \ge 0.02$,
the median trend shifts toward larger $T$ as communication becomes weaker,
supporting the monotonic trend predicted by our theory at the aggregate level.
At extremely sparse communication (e.g., $p=0.01$),
the estimate becomes higher-variance and may deviate from the trend,
which is consistent with the instability of argmax-based selection.
Dataset-wise optimal intervals and aggregate statistics are reported in
Appendix~\ref{tab:bestT_raw}.

\begin{figure}[t]
  \centering
  \includegraphics[width=0.9\linewidth]{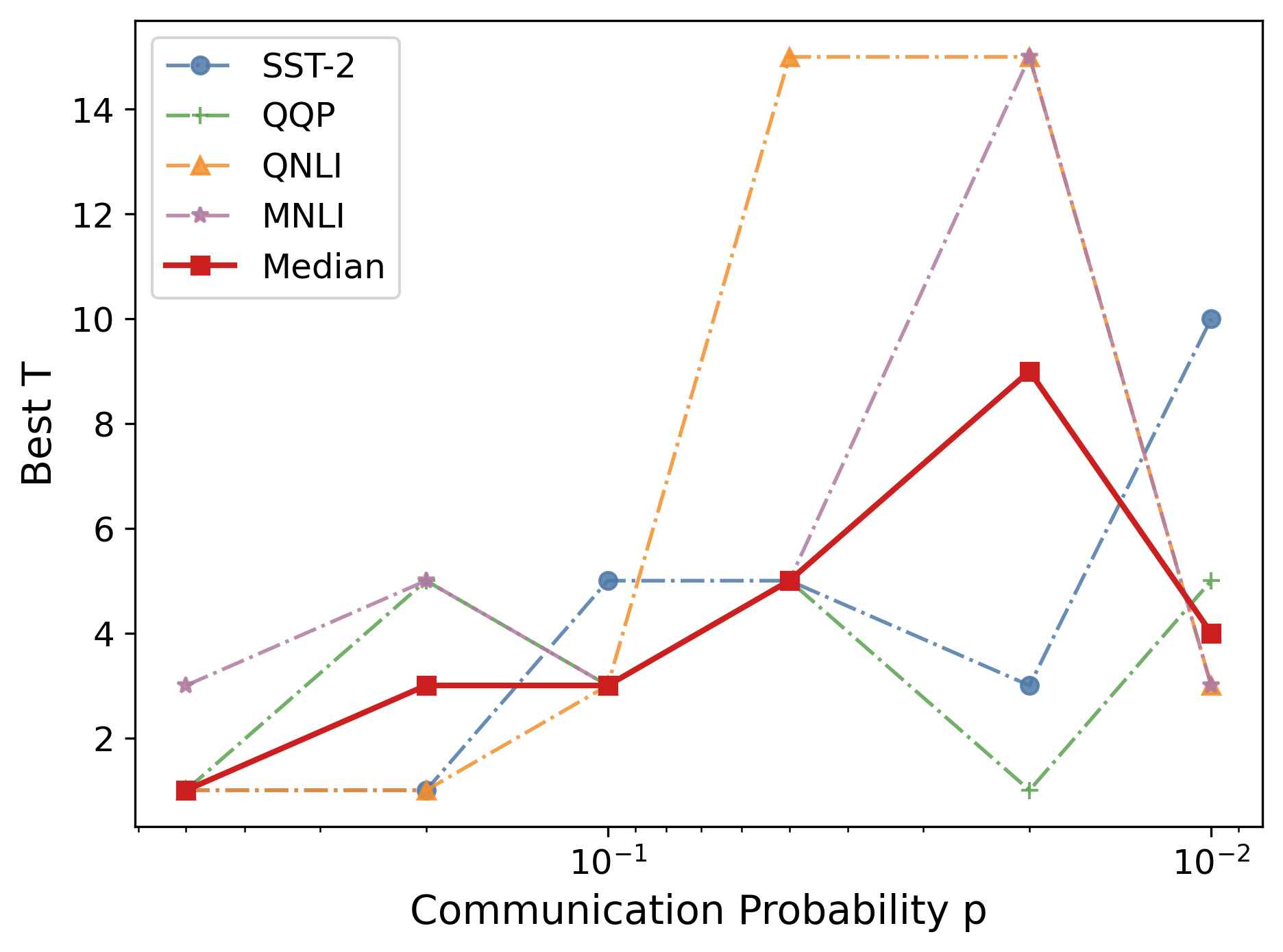}
\caption{
Dataset-wise optimal switching interval $\hat T^*(p)$ (thin dashed lines)
and the \textbf{Median} trend across datasets (thick line).
The median aggregation mitigates the high variance introduced by
argmax-based selection over a discrete candidate set.
In the reliably convergent regime $p \ge 0.02$,
the \textbf{Median} trend shifts toward larger $T$ as communication becomes weaker,
consistent with the monotonicity predicted by our theory.
}
\label{fig:bestT_median}
\end{figure}

\noindent{\bf From aggregate trends to instance-level behavior.}
The median optimal switching interval in Figure~\ref{fig:bestT_median}
reveals a clear population-level trend:
as communication becomes weaker ($p$ decreases),
the optimal switching interval shifts toward larger values.
To further understand this behavior at a finer granularity,
Figure~\ref{fig:mnli_heatmap} shows the accuracy gain of \texttt{TAD-LoRA} over the \texttt{LoRA} baseline
on MNLI across different communication probabilities $p$ and switching intervals $T$.
Under weak communication (small $p$),
performance gains concentrate on moderate to larger $T$ and span a wider range
of effective switching intervals.
As communication becomes stronger, the sensitivity to $T$ diminishes,
indicating improved robustness to the choice of the switching interval.

\begin{figure}[t]
  \centering
  \includegraphics[width=\linewidth]{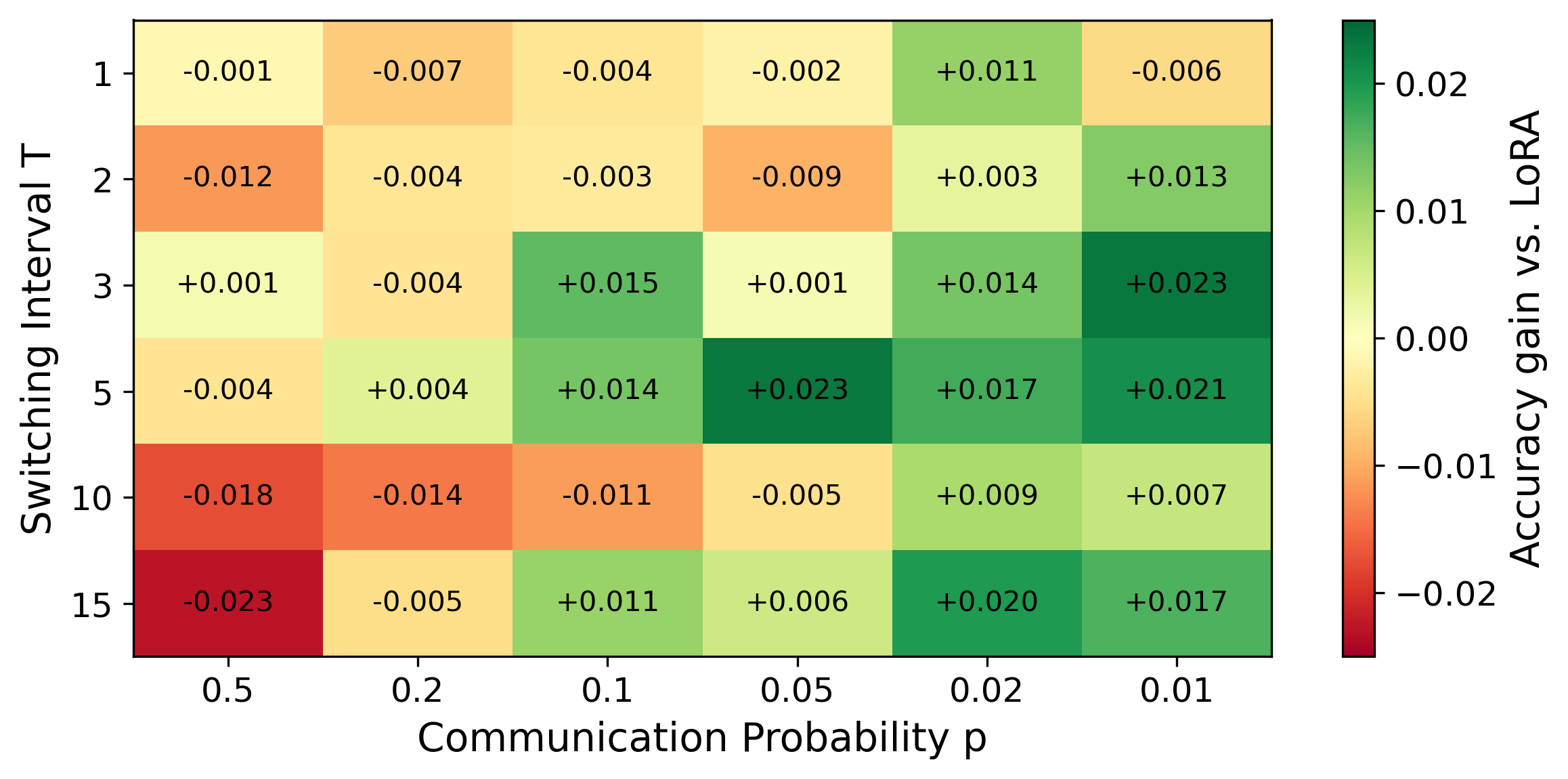}
  \caption{
  Accuracy gain of \texttt{TAD-LoRA} over the \texttt{LoRA} baseline on MNLI
  under different communication probabilities $p$ and switching intervals $T$.
  Positive values indicate performance improvement over \texttt{LoRA}.
  Under weak communication (small $p$),
  performance gains concentrate on moderate to larger $T$ and span a wider range
  of effective switching intervals.
  }
  \label{fig:mnli_heatmap}
\end{figure}

Additional result are deferred to the appendix.

\subsection{Switching Interval Selection in Practice}
In our experiments, the best switching intervals are selected in hindsight
to characterize the performance landscape and isolate the effect of $T$.
Importantly, our results show that under weak communication,
a broad range of switching intervals yields consistent improvements,
indicating robustness to imperfect $T$ selection.
A simple practical heuristic is to use smaller $T$ under strong communication
and gradually increase $T$ as communication becomes weaker.
Designing adaptive online strategies for selecting $T$ without oracle access
is an interesting direction for future work.

\noindent{\bf Summary of Findings.}
Across all experiments, we observe a consistent pattern: topology-aware
alternating updates provide limited benefit under strong communication, but
become increasingly advantageous as communication becomes sparse.
Crucially, the optimal switching interval is not fixed, but varies systematically
with the underlying communication topology.
As communication becomes weaker, larger switching intervals are preferred,
which aligns with our theoretical analysis of topology-induced errors and
alternation bias.
These results highlight the importance of jointly considering communication
topology and switching schedules in decentralized parameter-efficient tuning.

\section{Conclusion}
In this work, we introduced \texttt{TAD-LoRA}, a topology-aware alternating update
framework for decentralized federated fine-tuning. Unlike existing alternating LoRA methods that treat the switching interval as a
fixed or monotonic design choice, we identified a previously underexplored
trade-off in decentralized settings: overly frequent alternation leads to 
misalignment between LoRA blocks, while overly infrequent alternation amplifies
cross-term noise under stale and sparse communication. We showed that this trade-off induces a non-monotonic dependence on the switching interval $T$, fundamentally distinguishing decentralized alternating LoRA from its centralized counterpart.
We further provided theoretical analysis  that characterize how communication
topology and sparsity influence the convergence and switching interval selection, revealing that weaker communication favors larger switching intervals. These insights provide useful guidance for the design of practical
topology-aware switching strategies in decentralized fine-tuning.

Extensive experiments on GLUE benchmarks validate our analysis.
Across a wide range of communication scenarios, \texttt{TAD-LoRA} consistently
matches or outperforms decentralized baselines, and outperforms prior
alternating methods that are designed for centralized training.
Notably, the performance gains become increasingly pronounced as communication
becomes sparse, reflecting the fundamental mismatch between centralized
alternation strategies and decentralized communication dynamics,  and demonstrating that
topology-aware coordination yields robust and effective performance in
decentralized fine-tuning.

\noindent{\bf Future Work.}
Several promising directions remain.
First, while this work focuses on topology-aware switching strategies,
our theoretical analysis suggests the potential of adaptive switching policies
that adjust $T$ online based on communication conditions or observed training dynamics.

Second, extending \texttt{TAD-LoRA} to larger language models and longer training
horizons would further assess its scalability and practical impact.
Finally, exploring more structured and realistic communication patterns beyond
random and regular topologies remains an important direction for future research.

\bibliographystyle{IEEEtran}
\bibliography{reference}

@inproceedings{mcmahan2017communication,
  title={Communication-efficient learning of deep networks from decentralized data},
  author={McMahan, Brendan and Moore, Eider and Ramage, Daniel and Hampson, Seth and y Arcas, Blaise Aguera},
  booktitle={Artificial intelligence and statistics},
  pages={1273--1282},
  year={2017},
  organization={PMLR}
}

@article{li2020federated,
  title={Federated learning: Challenges, methods, and future directions},
  author={Li, Tian and Sahu, Anit Kumar and Talwalkar, Ameet and Smith, Virginia},
  journal={IEEE signal processing magazine},
  volume={37},
  number={3},
  pages={50--60},
  year={2020},
  publisher={IEEE}
}

@article{lian2017can,
  title={Can decentralized algorithms outperform centralized algorithms? a case study for decentralized parallel stochastic gradient descent},
  author={Lian, Xiangru and Zhang, Ce and Zhang, Huan and Hsieh, Cho-Jui and Zhang, Wei and Liu, Ji},
  journal={Advances in neural information processing systems},
  volume={30},
  year={2017}
}

@article{lalitha2019peer,
  title={Peer-to-peer federated learning on graphs},
  author={Lalitha, Anusha and Kilinc, Osman Cihan and Javidi, Tara and Koushanfar, Farinaz},
  journal={arXiv preprint arXiv:1901.11173},
  year={2019}
}

@article{zhao2018federated,
  title={Federated learning with non-iid data},
  author={Zhao, Yue and Li, Meng and Lai, Liangzhen and Suda, Naveen and Civin, Damon and Chandra, Vikas},
  journal={arXiv preprint arXiv:1806.00582},
  year={2018}
}

@article{kairouz2021advances,
  title={Advances and open problems in federated learning},
  author={Kairouz, Peter and McMahan, H Brendan and Avent, Brendan and Bellet, Aur{\'e}lien and Bennis, Mehdi and Bhagoji, Arjun Nitin and Bonawitz, Kallista and Charles, Zachary and Cormode, Graham and Cummings, Rachel and others},
  journal={Foundations and trends{\textregistered} in machine learning},
  volume={14},
  number={1--2},
  pages={1--210},
  year={2021},
  publisher={Now Publishers, Inc.}
}

@article{sattler2019robust,
  title={Robust and communication-efficient federated learning from non-iid data},
  author={Sattler, Felix and Wiedemann, Simon and M{\"u}ller, Klaus-Robert and Samek, Wojciech},
  journal={IEEE transactions on neural networks and learning systems},
  volume={31},
  number={9},
  pages={3400--3413},
  year={2019},
  publisher={IEEE}
}

@inproceedings{aji2017sparse,
  title={Sparse Communication for Distributed Gradient Descent},
  author={Aji, Alham Fikri and Heafield, Kenneth},
  booktitle={Proceedings of the 2017 Conference on Empirical Methods in Natural Language Processing},
  pages={440--445},
  year={2017}
}

@article{fallah2020personalized,
  title={Personalized federated learning with theoretical guarantees: A model-agnostic meta-learning approach},
  author={Fallah, Alireza and Mokhtari, Aryan and Ozdaglar, Asuman},
  journal={Advances in neural information processing systems},
  volume={33},
  pages={3557--3568},
  year={2020}
}

@article{xie2019asynchronous,
  title={Asynchronous federated optimization},
  author={Xie, Cong and Koyejo, Sanmi and Gupta, Indranil},
  journal={arXiv preprint arXiv:1903.03934},
  year={2019}
}

@inproceedings{wang2025decentralized,
  title={Decentralized federated learning with model caching on mobile agents},
  author={Wang, Xiaoyu and Xiong, Guojun and Cao, Houwei and Li, Jian and Liu, Yong},
  booktitle={Proceedings of the AAAI Conference on Artificial Intelligence},
  volume={39},
  number={20},
  pages={21296--21303},
  year={2025}
}

@inproceedings{sunimproving,
  title={Improving LoRA in Privacy-preserving Federated Learning},
  author={Sun, Youbang and Li, Zitao and Li, Yaliang and Ding, Bolin},
  booktitle={The Twelfth International Conference on Learning Representations}
}

@inproceedings{hulora,
  title={LoRA: Low-Rank Adaptation of Large Language Models},
  author={Hu, Edward J and Wallis, Phillip and Allen-Zhu, Zeyuan and Li, Yuanzhi and Wang, Shean and Wang, Lu and Chen, Weizhu and others},
  booktitle={International Conference on Learning Representations}
}

@article{touvron2023llama,
  title={Llama: Open and efficient foundation language models},
  author={Touvron, Hugo and Lavril, Thibaut and Izacard, Gautier and Martinet, Xavier and Lachaux, Marie-Anne and Lacroix, Timoth{\'e}e and Rozi{\`e}re, Baptiste and Goyal, Naman and Hambro, Eric and Azhar, Faisal and others},
  journal={arXiv preprint arXiv:2302.13971},
  year={2023}
}

@article{achiam2023gpt,
  title={Gpt-4 technical report},
  author={Achiam, Josh and Adler, Steven and Agarwal, Sandhini and Ahmad, Lama and Akkaya, Ilge and Aleman, Florencia Leoni and Almeida, Diogo and Altenschmidt, Janko and Altman, Sam and Anadkat, Shyamal and others},
  journal={arXiv preprint arXiv:2303.08774},
  year={2023}
}

@article{liu2024deepseek,
  title={Deepseek-v3 technical report},
  author={Liu, Aixin and Feng, Bei and Xue, Bing and Wang, Bingxuan and Wu, Bochao and Lu, Chengda and Zhao, Chenggang and Deng, Chengqi and Zhang, Chenyu and Ruan, Chong and others},
  journal={arXiv preprint arXiv:2412.19437},
  year={2024}
}

@inproceedings{zengglm,
  title={GLM-130B: An Open Bilingual Pre-trained Model},
  author={Zeng, Aohan and Liu, Xiao and Du, Zhengxiao and Wang, Zihan and Lai, Hanyu and Ding, Ming and Yang, Zhuoyi and Xu, Yifan and Zheng, Wendi and Xia, Xiao and others},
  booktitle={The Eleventh International Conference on Learning Representations}
}

@inproceedings{wangflora,
  title={FLoRA: Federated Fine-Tuning Large Language Models with Heterogeneous Low-Rank Adaptations},
  author={Wang, Ziyao and Shen, Zheyu and He, Yexiao and Sun, Guoheng and Wang, Hongyi and Lyu, Lingjuan and Li, Ang},
  booktitle={The Thirty-eighth Annual Conference on Neural Information Processing Systems}
}

@inproceedings{babakniya2023slora,
  title={SLoRA: Federated Parameter Efficient Fine-Tuning of Language Models},
  author={Babakniya, Sara and Elkordy, Ahmed Roushdy and Ezzeldin, Yahya H and Liu, Qingfeng and Song, Kee-Bong and EL-Khamy, MOSTAFA and Avestimehr, Salman},
  booktitle={International Workshop on Federated Learning in the Age of Foundation Models in Conjunction with NeurIPS 2023}
}

@inproceedings{
guo2025selective,
title={Selective Aggregation for Low-Rank Adaptation in Federated Learning},
author={Pengxin Guo and Shuang Zeng and Yanran Wang and Huijie Fan and Feifei Wang and Liangqiong Qu},
booktitle={The Thirteenth International Conference on Learning Representations},
year={2025},
url={https://openreview.net/forum?id=iX3uESGdsO}
}

@misc{yi2024pfedloramodelheterogeneouspersonalizedfederated,
      title={pFedLoRA: Model-Heterogeneous Personalized Federated Learning with LoRA Tuning}, 
      author={Liping Yi and Han Yu and Gang Wang and Xiaoguang Liu and Xiaoxiao Li},
      year={2024},
      eprint={2310.13283},
      archivePrefix={arXiv},
      primaryClass={cs.LG},
      url={https://arxiv.org/abs/2310.13283}, 
}

@inproceedings{
bai2024federated,
title={Federated Fine-tuning of Large Language Models under Heterogeneous Tasks and Client Resources},
author={Jiamu Bai and Daoyuan Chen and Bingchen Qian and Liuyi Yao and Yaliang Li},
booktitle={The Thirty-eighth Annual Conference on Neural Information Processing Systems},
year={2024},
url={https://openreview.net/forum?id=gkOzoHBXUw}
}

@article{chen2025robust,
  title={Robust federated finetuning of llms via alternating optimization of lora},
  author={Chen, Shuangyi and Guo, Yuanxin and Ju, Yue and Dalal, Harik and Zhu, Zhongwen and Khisti, Ashish},
  journal={arXiv preprint arXiv:2502.01755},
  year={2025}
}

@article{yu2025altlora,
  title={AltLoRA: Towards Better Gradient Approximation in Low-Rank Adaptation with Alternating Projections},
  author={Yu, Xin and Wang, Yujia and Chen, Jinghui and Xue, Lingzhou},
  journal={arXiv preprint arXiv:2505.12455},
  year={2025}
}

@article{erdds1959random,
  title={On random graphs I},
  author={ERDdS, P and R\&wi, A},
  journal={Publ. math. debrecen},
  volume={6},
  number={290-297},
  pages={18},
  year={1959}
}

\onecolumn
\appendices
\section{Full Theoretical Analysis}
\label{appendix:full-theory}

This appendix provides the complete convergence analysis supporting Section~\ref{theoretical_analysis}. While the main text highlights the key trade-off, here we explicitly detail the role of the LoRA product structure, the cross-term induced by decentralized aggregation, and the blockwise optimization bias caused by alternating updates with period~$T$.
\\
\noindent{\bf Roadmap.} Our analysis proceeds through the following key components:
\begin{enumerate}
    \item \textbf{Per-block consensus:} Establishing that frozen blocks contract purely via gossip, while updated blocks incur bounded error.
    \item \textbf{Decomposition:} Splitting the averaged LoRA update into a "centralized-equivalent" factor term and a topology-dependent cross term.
    \item \textbf{Cross-term decay:} Proving that the accumulated topology error decreases as $O(1/T)$ within each alternating cycle.
    \item \textbf{Representation bias:} Quantifying the drift caused by optimizing against stale parameters, which increases with $T$.
    \item \textbf{Main Convergence:} Combining these effects to derive the stationarity bound and the topology-dependent optimal switching period $T^\star$.
\end{enumerate}

\subsection{Notation and Preliminaries}\label{app:notation}

\noindent{\bf Norm convention.} Unless otherwise specified, $\|\cdot\|$ denotes the Euclidean norm for vectors and the Frobenius norm for matrices.

\noindent{\bf LoRA Formulation.} Each client~$i$ holds LoRA parameters $(A_i^t,B_i^t)$ and the local model is $\theta_i^t = \theta_0 + B_i^t A_i^t$.We denote global averages as:$$  \bar A^t = \tfrac1m\sum_{i=1}^m A_i^t,
  \qquad
  \bar B^t = \tfrac1m\sum_{i=1}^m B_i^t,
  \qquad
  \bar\theta^t = \tfrac1m\sum_{i=1}^m \theta_i^t .$$
  
\noindent{\bf Communication Topology.} Let $W_t$ be the random doubly-stochastic mixing matrix. We assume the network satisfies the spectral gap property:$$  \mathbb E\Bigl\| W_t - \tfrac1m {\bf 1}{\bf 1}^\top \Bigr\|_2^2 \le \rho^2,
  \qquad
  0 < \rho < 1.$$

\noindent{\bf Reference Optimum ($T=1$).} \label{appendix_reference_optimum}To quantify the bias of coarse-grained switching ($T>1$), we define the \emph{reference optimum} $\theta_1^\star$ as the optimal solution attained by an ideal centralized alternating LoRA method with $T=1$ (switching every step). Let $F_1^\star := F(\theta_1^\star)$. This serves as the baseline for our suboptimality analysis.

\subsection{Assumptions} \label{app:assumptions}
We first state standard regularity assumptions commonly used in decentralized stochastic optimization.

\begin{assumption}[Regularity Conditions] \label{ass:standard}
The following conditions hold:
    \begin{itemize}
        \item \textbf{L-Smoothness:} Each local objective $f_i$ is $L$-smooth.
        \item \textbf{Stochastic Gradients:} Local stochastic gradients are unbiased and have bounded second moments (variance $\sigma^2$).
        \item \textbf{Mixing Contraction:} The mixing matrices $W_t$ satisfy the mean-square contraction property defined in Section~\ref{app:notation}.
    \end{itemize}
\end{assumption}

In addition to standard regularity, our analysis requires two structural assumptions specific to the alternating LoRA landscape, as introduced in the main text.

\begin{assumption}[Local PL on the LoRA subspace]\label{ass:pl-appendix}(Restatement of Assumption~\ref{ass:pl-main}). 
There exists $\mu>0$ such that for all $\theta$ in a neighborhood of $\theta_1^\star$,$$  F(\theta) - F_1^\star \le \frac{1}{2\mu}\|\nabla F(\theta)\|^2.$$
\end{assumption}

\begin{assumption}[Alternating-induced factorization bias]\label{ass:bias-appendix}(Restatement of Assumption~\ref{ass:alt-bias}). Let $F_T^\star$ denote the optimal value achievable with switching interval $T$. There exists $C_3>0$ such that for sufficiently small stepsize $\eta$:$$  \phi(T)
  :=
  F_T^\star - F_1^\star
  \le
  C_3\,T\eta^2.$$
\end{assumption}

\subsection{Per-Block Consensus Under Alternating Updates}

We analyze the consensus error $\Delta_A^t := A^t - \bar A^t$.
\begin{lemma}[Per-block consensus]\label{lem:block-consensus-appendix}
Under the standing assumptions (spectral gap $\rho$ and bounded gradient variance $\sigma^2$), for stepsize $\eta$:
    \begin{itemize}
        \item \textbf{Updated Block:} $\mathbb E\|\Delta_A^{t+1}\|^2 \le \frac{1+\rho^2}{2}\,\mathbb E\|\Delta_A^t\|^2 + C_{\mathrm{blk}}\eta^2$.
        \item \textbf{Frozen Block:} $\mathbb E\|\Delta_A^{t+1}\|^2 \le \rho^2 \mathbb E\|\Delta_A^t\|^2$.
    \end{itemize}
\end{lemma}

\begin{proof}
\textbf{Case 1: Frozen Block.} The update is purely communicative: $A_i^{t+1} = \sum_j W_{ij} A_j^t$. In matrix form, this is $\Delta_A^{t+1} = (W_t - \frac{1}{m}\mathbf{1}\mathbf{1}^\top) \Delta_A^t$. Taking the norm and applying the spectral gap definition directly yields $\mathbb E\|\Delta_A^{t+1}\|^2 \le \rho^2 \mathbb E\|\Delta_A^t\|^2$.
\textbf{Case 2: Updated Block.} The update includes a gradient term: $A_i^{t+1} = \sum_j W_{ij} A_j^t - \eta g_i^t$. The error decomposes into a mixing part and a noise part:$\Delta_A^{t+1} = (W_t - \frac{1}{m}\mathbf{1}\mathbf{1}^\top) \Delta_A^t - \eta (I - \frac{1}{m}\mathbf{1}\mathbf{1}^\top) g^t.$ Using the inequality $\|X+Y\|^2 \le (1+\gamma)\|X\|^2 + (1+\frac{1}{\gamma})\|Y\|^2$ to separate the terms, and bounding the gradient variance, we obtain the stated result.
\end{proof}

\noindent{\bf Steady-state disagreement.}
From Lemma~\ref{lem:block-consensus-appendix},
$u_{t+1}\le \rho^2 u_t + C_{\mathrm{blk}}\eta^2G^2$ with
$u_t:=\mathbb E\|\Delta_B^t\|^2$, hence
$u_t \le
  \rho^{2t}u_0
  + C_{\mathrm{blk}}\eta^2G^2\sum_{s=0}^{t-1}\rho^{2s} \le \rho^{2t}u_0 + \frac{C_{\mathrm{blk}}\eta^2G^2}{1-\rho^2}
= O(\eta^2/(1-\rho))$ after a transient, and similarly for $\Delta_A^t$.
\subsection{Cross-Term Decomposition}\label{sec:cross-term} We decompose the global update to isolate topology-induced errors. The averaged LoRA update matrix is $\bar W^t := \tfrac1m \sum_{i=1}^m B_i^t A_i^t$. Adding and subtracting the product of averages $\bar B^t \bar A^t$ yields:$$  \bar W^t
  = \bar B^t \bar A^t
    +
    C^t ,
  \qquad \text{where} \quad
  C^t := \tfrac1m \sum_{i=1}^m (B_i^t-\bar B^t)(A_i^t-\bar A^t).$$By Cauchy--Schwarz, the norm of this cross term is bounded by the product of disagreements:$$  \|C^t\|_F
  \le
  \sqrt{\tfrac1m\sum \|B_i^t-\bar B^t\|^2} \cdot \sqrt{\tfrac1m\sum \|A_i^t-\bar A^t\|^2}
  = \|\Delta_A^t\|\;\|\Delta_B^t\|.$$

\noindent{\bf Effect on Optimization.}
 Assuming the objective function is $G$-Lipschitz (bounded gradients), the perturbation in the objective value caused by $C^t$ is bounded by:$$  |F(\bar\theta^t) - F(\theta_0 + \bar B^t \bar A^t)|
  \le
  G \|C^t\|_F
  \le
  G \|\Delta_A^t\|\;\|\Delta_B^t\|.$$This explicitly connects the topology error to the optimization objective: minimizing blockwise disagreement minimizes the deviation from the centralized trajectory.

\subsection{Cross-Term Decay Within an Alternating Cycle}\label{sec:cross-term-decay}

\begin{proposition}[Cycle-averaged cross-term decay]
\label{prop:cycle-cross}
For any alternating cycle of length $T$,
\[
  \frac1T\sum_{\tau=0}^{T-1}\mathbb E\|C^{t+\tau}\|_F
  \;\le\;
  \frac{C_{\mathrm{cr}}\eta^2}{T(1-\rho)}
\]
for a constant $C_{\mathrm{cr}}>0$ independent of $T$ and $\rho$.
\end{proposition}

\noindent{\bf Proof.}
Recall $\|C^t\|_F \le \|\Delta_A^t\|\,\|\Delta_B^t\|$.
Consider a phase of length $T$ in which block $A$ is frozen and block $B$ is updated.
By Lemma~\ref{lem:block-consensus-appendix}, for $\tau=0,\dots,T-1$,
\[
  \mathbb E\|\Delta_A^{t+\tau}\|^2 \le \rho^{2\tau}\,\mathbb E\|\Delta_A^t\|^2.
\]
Meanwhile, during the update phase, the disagreement of the updated block reaches
a steady-state level $\mathbb E\|\Delta_B^{t+\tau}\|^2 \le C\,\eta^2/(1-\rho)$.
Using Cauchy--Schwarz and Jensen,
\[
  \mathbb E\|C^{t+\tau}\|_F
  \le \sqrt{\mathbb E\|\Delta_A^{t+\tau}\|^2}\,\sqrt{\mathbb E\|\Delta_B^{t+\tau}\|^2}
  \le \rho^{\tau}\sqrt{\mathbb E\|\Delta_A^t\|^2}\cdot \sqrt{\frac{C\eta^2}{1-\rho}}.
\]
Averaging over $\tau$ yields a geometric series:
\[
  \frac1T\sum_{\tau=0}^{T-1}\mathbb E\|C^{t+\tau}\|_F
  \le
  \frac{1}{T}\left(\sum_{\tau=0}^{T-1}\rho^\tau\right)
  \sqrt{\mathbb E\|\Delta_A^t\|^2}\cdot \sqrt{\frac{C\eta^2}{1-\rho}}.
\]
After the transient, Lemma~\ref{lem:block-consensus-appendix} implies
$\mathbb E\|\Delta_A^t\|^2 \le C'\eta^2/(1-\rho)$. Substituting and using
$\sum_{\tau=0}^{T-1}\rho^\tau \le \frac{1}{1-\rho}$ gives
\[
  \frac1T\sum_{\tau=0}^{T-1}\mathbb E\|C^{t+\tau}\|_F
  \le
  \frac{C_{\mathrm{cr}}\eta^2}{T(1-\rho)},
\]
which proves the claim. This shows that the cycle-averaged cross term decreases as $O(1/T)$.

\subsection{One-step Descent of the Averaged Model}
\begin{lemma}[One-step descent of the averaged model]
\label{lem:descent-avg}
Suppose $F$ is $L$-smooth and the averaged iterate satisfies
$\bar\theta^{t+1}=\bar\theta^t-\eta g^t$, where $g^t=\frac1m\sum_i g_i^t$.
Then for sufficiently small $\eta$,
\[
  \mathbb E[F(\bar\theta^{t+1})]
  \le
  \mathbb E[F(\bar\theta^t)]
  -
  \eta\,\mathbb E\|\nabla F(\bar\theta^t)\|^2
  +
  C\eta^2
  \bigl(
     1
     + \mathbb E\|\Delta_A^t\|^2
     + \mathbb E\|\Delta_B^t\|^2
     + \mathbb E\|C^t\|_F
  \bigr).
\]
\end{lemma}

\begin{proof}
By $L$-smoothness (descent lemma),
\[
  F(\bar\theta^{t+1})
  \le
  F(\bar\theta^t)
  -\eta\langle \nabla F(\bar\theta^t), g^t\rangle
  +\frac{L\eta^2}{2}\|g^t\|^2.
\]
Take conditional expectation w.r.t.\ the randomness at time $t$ and use
$\mathbb E[g^t \mid \mathcal F_t]=\tfrac1m\sum_i\nabla f_i(\theta_i^t)$.

\noindent{\bf Gradient decomposition.}
Write
\[
  g^t
  =
  \nabla F(\bar\theta^t)
  +
  e_{\mathrm{cons}}^t
  +
  e_{\mathrm{noise}}^t,
\]
where
\[
  e_{\mathrm{cons}}^t
  :=
  \frac1m\sum_{i=1}^m\bigl(\nabla f_i(\theta_i^t)-\nabla f_i(\bar\theta^t)\bigr),
  \qquad
  e_{\mathrm{noise}}^t
  :=
  \frac1m\sum_{i=1}^m\bigl(g_i^t-\nabla f_i(\theta_i^t)\bigr).
\]
By unbiased stochastic gradients and bounded variance,
\[
  \mathbb E\!\left[e_{\mathrm{noise}}^t \mid \mathcal F_t\right]=0,
  \qquad
  \mathbb E\!\left[\|e_{\mathrm{noise}}^t\|^2 \mid \mathcal F_t\right]\le \frac{\sigma^2}{m}
  \le C.
\]

\noindent{\bf Bounding the consensus/heterogeneity term via LoRA factor disagreements and the cross term.}
Using $L$-smoothness of each $f_i$,
\[
  \|e_{\mathrm{cons}}^t\|
  \le
  \frac{1}{m}\sum_{i=1}^m
  \bigl\|\nabla f_i(\theta_i^t)-\nabla f_i(\bar\theta^t)\bigr\|
  \le
  \frac{L}{m}\sum_{i=1}^m
  \|\theta_i^t-\bar\theta^t\|.
\]
Now decompose the model deviation using the LoRA product structure:
\[
  \theta_i^t-\bar\theta^t
  =
  (\theta_0+B_i^tA_i^t)-(\theta_0+\bar B^t\bar A^t + C^t)
  =
  \underbrace{(B_i^tA_i^t-\bar B^t\bar A^t)}_{\text{factor disagreement}}
  \;-\;
  \underbrace{C^t}_{\text{cross term}}.
\]
Hence,
\[
  \|\theta_i^t-\bar\theta^t\|
  \le
  \|B_i^tA_i^t-\bar B^t\bar A^t\|
  +\|C^t\|_F.
\]
Applying $\|XY\|_F\le \|X\|_F\|Y\|_2$ and standard bilinear expansion yields
(for bounded factor norms, absorbed into the constant)
\[
  \frac{1}{m}\sum_{i=1}^m \|B_i^tA_i^t-\bar B^t\bar A^t\|^2
  \le
  C\bigl(\|\Delta_A^t\|^2+\|\Delta_B^t\|^2\bigr).
\]
Combining the above bounds and Jensen's inequality gives
\[
  \mathbb E\|e_{\mathrm{cons}}^t\|^2
  \le
  C\Bigl(\mathbb E\|\Delta_A^t\|^2+\mathbb E\|\Delta_B^t\|^2+\mathbb E\|C^t\|_F\Bigr),
\]
where we used $\|C^t\|_F^2\le \|C^t\|_F$ after absorbing constants (or equivalently
replace $\mathbb E\|C^t\|_F$ by $\mathbb E\|C^t\|_F^2$ if you prefer a squared term).

\noindent{\bf Putting pieces together.}
Plugging $g^t=\nabla F(\bar\theta^t)+e_{\mathrm{cons}}^t+e_{\mathrm{noise}}^t$ into the
descent lemma and expanding,
\[
  -\eta\langle \nabla F(\bar\theta^t), g^t\rangle
  =
  -\eta\|\nabla F(\bar\theta^t)\|^2
  -\eta\langle \nabla F(\bar\theta^t), e_{\mathrm{cons}}^t\rangle
  -\eta\langle \nabla F(\bar\theta^t), e_{\mathrm{noise}}^t\rangle.
\]
Apply Young's inequality
$\langle a,b\rangle \le \tfrac{1}{4}\|a\|^2+\|b\|^2$ to the two inner products to obtain
\[
  -\eta\langle \nabla F(\bar\theta^t), e_{\mathrm{cons}}^t\rangle
  \le
  \frac{\eta}{4}\|\nabla F(\bar\theta^t)\|^2
  +\eta\|e_{\mathrm{cons}}^t\|^2,
  \qquad
  -\eta\langle \nabla F(\bar\theta^t), e_{\mathrm{noise}}^t\rangle
  \le
  \frac{\eta}{4}\|\nabla F(\bar\theta^t)\|^2
  +\eta\|e_{\mathrm{noise}}^t\|^2.
\]
Similarly, $\|g^t\|^2 \le C\bigl(\|\nabla F(\bar\theta^t)\|^2+\|e_{\mathrm{cons}}^t\|^2+\|e_{\mathrm{noise}}^t\|^2\bigr)$.
Taking expectation and choosing $\eta$ sufficiently small so that the
$\eta^2\mathbb E\|\nabla F(\bar\theta^t)\|^2$ term can be absorbed into the main
descent term yields
\[
  \mathbb E[F(\bar\theta^{t+1})]
  \le
  \mathbb E[F(\bar\theta^t)]
  -
  \eta\,\mathbb E\|\nabla F(\bar\theta^t)\|^2
  +
  C\eta^2\Bigl(
     1
     +\mathbb E\|\Delta_A^t\|^2
     +\mathbb E\|\Delta_B^t\|^2
     +\mathbb E\|C^t\|_F
  \Bigr),
\]
which proves the claim.
\end{proof}

\subsection{Topology-Induced Stationarity Bound}

We now combine the one-step descent result in
Lemma~\ref{lem:descent-avg}
with steady-state block disagreements and the cycle-averaged cross-term decay
(Proposition~\ref{prop:cycle-cross}) to derive a stationarity bound.
\begin{theorem}[Stationarity under alternating LoRA]
\[
  \frac1R\sum_{t=0}^{R-1}
  \mathbb E\|\nabla F(\bar\theta^t)\|^2
  \le
  \frac{C_0}{\eta R}
  + C_1\eta
  + \frac{C_2\eta^2}{T(1-\rho)}
  + O(\eta^2).
\]
The last term $\frac{C_2\eta^2}{T(1-\rho)}$ corresponds exactly to the
cross-term decay in Section~\ref{sec:cross-term-decay}.
\end{theorem}

\noindent{\bf Representation bias induced by coarse switching.}
We now interpret the term $\phi(T)$ introduced in
Assumption~\ref{ass:bias-appendix}.
This assumption captures the inherent cost of freezing one LoRA factor for
$T$ consecutive updates.
When one block is optimized against a stale counterpart, a deterministic
approximation error is introduced.
Under standard smoothness, this mismatch incurs a per-step error of order
$O(\eta^2)$, which accumulates linearly over the freezing horizon~$T$.
As a result,
\[
  \phi(T)
  =
  F_T^\star - F_1^\star
  \;\le\;
  C_3 T \eta^2.
\]
This term reflects a representation-level bias that is independent of the
communication topology, and it increases monotonically with the switching
interval~$T$.

\noindent{\bf PL condition (restatement of Assumption~\ref{ass:pl-main}).}
We assume the local PL condition stated in Assumption~\ref{ass:pl-main}.

\noindent{\bf Remark.}
The PL condition is imposed only locally around the reference optimum
$\theta_1^\star$ and is used to relate gradient stationarity to function-value
suboptimality. It does not require the algorithm to operate with switching
interval $T=1$, nor does it restrict the optimization trajectory to follow
the $T=1$ alternating scheme.
\subsection{Suboptimality relative to the reference optimum}
\label{thm:subopt}

\begin{theorem}[Function-value gap]
\[
  \mathbb E[F(\hat\theta_R)-F^\star_1]
  \le
  \frac{1}{2\mu}
  \left(
    \frac{C_0}{\eta R}
    + C_1\eta
    + \frac{C_2}{T(1-\rho)}
  \right)
  +
  \phi(T).
\]
The term $C_2/[T(1-\rho)]$ comes from cross-term topology effects, while $\phi(T)$
captures block-coordinate bias from Assumption~\ref{ass:alt-bias}. $F^\star_1 := F(\theta_1^\star)$ denotes the objective value at the
reference optimum defined in Section~\ref{appendix_reference_optimum}.
Let $\hat\theta_R$ be chosen uniformly from $\{\bar\theta^0,\dots,\bar\theta^{R-1}\}$.
\end{theorem}

\subsection{Topology-Dependent Optimal Switching Period}

Define
\[
  \Psi(T;\rho)
  :=
  \frac{C_2}{T(1-\rho)} + \phi(T).
\]

\begin{corollary}[Optimal switching interval]
\[
  T^\star(\rho)
  \in
  \arg\min_{T\ge 1} \Psi(T;\rho).
\]
Under Assumption~\ref{ass:alt-bias},
\[
  \Psi(T;\rho)
  \le
  \frac{C_2}{T(1-\rho)} + C_3 T\eta^2,
\]
whose minimizer satisfies
\[
  T^\star(\rho)
  \in
  \Theta\!\left(
    \frac{1}{\sqrt{1-\rho}}
  \right).
\]
\end{corollary}

\noindent{\bf Interpretation.}
\[
\boxed{
\text{Small $T$: cross-term dominates }
\frac{1}{T(1-\rho)}.
\qquad
\text{Large $T$: block-coordinate bias dominates } \phi(T)\sim T\eta^2.
}
\]
The optimal $T^\star$ balances the two sources of error.

\subsection{Edge-activation gossip and spectral gap scaling}\label{appendix:ER}
\begin{lemma}[Edge-activation gossip implies a spectral gap scaling]
\label{lem:rho-from-p-lambda}

Consider a fixed connected undirected graph $G=(V,E)$ with $m:=|V|$ nodes and
(combinatorial) Laplacian $L$. At iteration $t$, each edge $e=(i,j)\in E$ is
activated independently with probability $p\in(0,1]$. For every activated edge,
a \emph{pairwise averaging} update is performed between its incident nodes:
\[
  x_i \leftarrow \tfrac12(x_i+x_j),\qquad
  x_j \leftarrow \tfrac12(x_i+x_j),
\]
and if a node participates in multiple activated edges, the corresponding
pairwise updates are applied in a uniformly random order within the iteration.
Let the resulting one-iteration mixing matrix be $W_t$ (so $x^{t+1}=W_t x^t$).
Then $W_t$ is doubly-stochastic and satisfies the mean-square contraction
assumption
\[
  \mathbb E\Big\|W_t-\tfrac1m{\bf 1}{\bf 1}^\top\Big\|_2^2 \le \rho^2
\]
for some $\rho\in(0,1)$. Moreover, there exists a constant $c_{\mathrm{mix}}>0$
(depending only on the averaging rule and the maximum degree of $G$, but
independent of $p,T,R$) such that the effective spectral gap obeys
\[
  1-\rho \;\ge\; c_{\mathrm{mix}}\; p\,\lambda_2(L).
\]

\end{lemma}

\begin{proof}[Proof sketch]
Let $J := \tfrac1m{\bf 1}{\bf 1}^\top$ and define the disagreement vector
$y := (I-J)x$ so that $y\perp {\bf 1}$. It suffices to control the contraction
of $\|y\|^2$ in one iteration. For a single activated edge $e=(i,j)$, the
pairwise averaging matrix can be written as
\[
  W_e = I - \tfrac12 L_e,
\]
where $L_e$ is the rank-1 Laplacian associated with edge $e$. When multiple
edges are activated, the within-iteration random-order update yields
\[
  W_t = \prod_{e\in E_t} W_e,
\]
where $E_t\subseteq E$ is the random set of activated edges. Using the standard
(Lie/Trotter-type) first-order approximation for products of near-identity
averaging operators, one obtains the Laplacian-averaging abstraction
\[
  W_t \approx I - \alpha L_t,
  \qquad
  L_t := \sum_{e\in E_t} L_e,
\]
for some constant step size $\alpha\in(0,1]$ determined by the averaging rule.
Since each edge is activated independently with probability $p$,
\[
  \mathbb E[L_t] = p \sum_{e\in E} L_e = pL,
  \qquad
  \Rightarrow\qquad
  \mathbb E[W_t] \approx I - \alpha p L.
\]
Because $L{\bf 1}=0$ and $L$ is symmetric PSD with eigenvalues
$0=\lambda_1(L)<\lambda_2(L)\le\cdots\le\lambda_m(L)$, we have
\[
  \lambda_2\!\big(\mathbb E[W_t]\big)
  \approx 1 - \alpha p \lambda_2(L),
\]
so the spectral gap scales linearly in $p\lambda_2(L)$. Finally, in symmetric
doubly-stochastic settings, the mean-square contraction
$\mathbb E\|W_t-J\|_2^2$ can be upper bounded (up to constants) by the second
eigenvalue of $\mathbb E[W_t^\top W_t]$, which inherits the same linear scaling
for sufficiently small $\alpha$. Absorbing approximation and higher-order terms
into $c_{\mathrm{mix}}$ yields $1-\rho \ge c_{\mathrm{mix}}\,p\,\lambda_2(L)$.
\end{proof}
\begin{corollary}[Optimal switching interval under independent edge activation]
\label{cor:optT-p}

Consider the independent edge activation model on a fixed underlying graph
$G=(V,E)$: at each iteration, every edge is activated independently with
probability $p\in(0,1]$, and activated edges perform pairwise averaging updates
(with a random order when a node participates in multiple activated edges).
Let $L$ be the (combinatorial) Laplacian of $G$ and $\lambda_2(L)$ be its
algebraic connectivity.

Assume the induced mixing matrix $W_t$ satisfies the mean-square contraction
assumption
\[
  \mathbb E\Big\|W_t-\tfrac{1}{m}{\bf 1}{\bf 1}^\top\Big\|_2^2 \le \rho^2,
\]
and moreover the effective spectral gap scales as
\[
  1-\rho \;\ge\; c_{\mathrm{mix}}\; p\,\lambda_2(L),
\]
for some constant $c_{\mathrm{mix}}>0$ determined by the averaging rule
(e.g., the effective step size in the Laplacian-averaging abstraction) and
independent of $p,T,R$.

Then the dominant $T$-dependent error term in
Theorem~\ref{thm:subopt} can be written as
\[
  \Psi(T; p, L)
  :=
  \frac{C_2}{T(1-\rho)} + \phi(T)
  \;\le\;
  \frac{\tilde C_2}{T\,p\,\lambda_2(L)} + \phi(T),
\]
where $\tilde C_2 := C_2/c_{\mathrm{mix}}$.

Consequently, an optimal switching interval satisfies
\[
  T^\star(p,L)
  \in
  \arg\min_{T\ge 1}
  \left\{
    \frac{\tilde C_2}{T\,p\,\lambda_2(L)} + \phi(T)
  \right\}.
\]
If in addition the alternating-induced bias obeys $\phi(T)\le C_3 T\eta^2$, then
balancing the two dominant terms yields
\[
T^\star(p,L) \in \Theta\!\left(\frac{1}{\sqrt{p\,\lambda_2(L)}}\right).
\]
In particular, better connectivity (larger $\lambda_2(L)$) and more reliable
communication (larger $p$) both lead to smaller optimal switching intervals.
\end{corollary}

\subsection{Scope of the analysis.}
All local regularity conditions (e.g., PL and smoothness on the LoRA subspace)
are imposed with respect to the reference optimum $\theta_1^\star$ and are
used solely to convert stationarity guarantees into function-value bounds.
They do not impose any restriction on the switching interval $T$ used by the
decentralized algorithm.

\clearpage

\section{Additional Experimental Results}

\subsection{Additional Communication Probabilities.}
Table~\ref{tab:appendix_full_results} reports additional results under communication
probabilities $p \in \{0.2, 0.1, 0.05, 0.01\}$, which are not included in the main
results due to space constraints.
The same evaluation protocol as in the main paper is used, reporting test
accuracy as mean $\pm$ variance over 3 random seeds.

These additional results provide a more complete picture of how different methods
behave across a wide range of communication regimes.
As the communication probability decreases, performance degradation is observed
for all methods, which is expected due to reduced model exchanges.
Nevertheless, \texttt{TAD-LoRA} consistently achieves the strongest performance
under low communication probabilities, not only in terms of the average across
datasets, but also on individual tasks.

Best and second-best results for each dataset under the same communication
probability are highlighted to emphasize per-dataset robustness.
This demonstrates that the gains of \texttt{TAD-LoRA} under sparse communication
are consistent across heterogeneous tasks, rather than being driven by a single
dataset.
\FloatBarrier
\begin{table*}[t]
\centering
\caption{Full results under communication probabilities.
Test accuracy is reported as mean $\pm$ variance over 3 random seeds.
Best and second-best results for each dataset under the same $p$ are highlighted
in bold and underline, respectively.
}
\label{tab:appendix_full_results}
\resizebox{\textwidth}{!}{
\begin{tabular}{c|l|cccc|c}
\toprule
$p$ & Method & SST-2 & QQP & QNLI & MNLI & Avg. \\
\midrule
\multirow{4}{*}{0.5}
 & \texttt{LoRA}
 & \textbf{0.9468} $\pm$ 0.0019
 & \textbf{0.8347} $\pm$ 0.0085
 & \textbf{0.9067} $\pm$ 0.0043
 & \underline{0.8132} $\pm$ 0.0194
 & \textbf{0.8754} \\
 & \texttt{FFA-LoRA}
 & 0.9436 $\pm$ 0.0032
 & 0.8051 $\pm$ 0.0061
 & 0.8911 $\pm$ 0.0036
 & 0.7313 $\pm$ 0.0333
 & 0.8428 \\
 & \texttt{RoLoRA}
 & \underline{0.9462} $\pm$ 0.0023
 & 0.8216 $\pm$ 0.0108
 & \underline{0.9021} $\pm$ 0.0031
 & 0.8115 $\pm$ 0.0088
 & 0.8703 \\
 & \texttt{TAD-LoRA} (Ours)
 & 0.9448 $\pm$ 0.0007
 & \underline{0.8328} $\pm$ 0.0072
 & 0.9003 $\pm$ 0.0048
 & \textbf{0.8145} $\pm$ 0.0066
 & \underline{0.8731} \\
\midrule
\multirow{4}{*}{0.2}
 &\texttt{LoRA}
 & \underline{0.9451} $\pm$ 0.0049
 & 0.8202 $\pm$ 0.0059
 & \underline{0.8950} $\pm$ 0.0117
 & \underline{0.7933} $\pm$ 0.0098
 & \underline{0.8634} \\
 & \texttt{FFA-LoRA}
 & 0.9365 $\pm$ 0.0025
 & 0.7996 $\pm$ 0.0047
 & 0.8859 $\pm$ 0.0032
 & 0.7300 $\pm$ 0.0285
 & 0.8380 \\
 &\texttt{ RoLoRA}
 & 0.9407 $\pm$ 0.0017
 & \underline{0.8119} $\pm$ 0.0142
 & 0.8826 $\pm$ 0.0065
 & 0.7693 $\pm$ 0.0140
 & 0.8511 \\
 & \texttt{TAD-LoRA} (Ours)
 & \textbf{0.9453} $\pm$ 0.0038
 & \textbf{0.8254} $\pm$ 0.0019
 & \textbf{0.8953} $\pm$ 0.0030
 & \textbf{0.7974} $\pm$ 0.0107
 & \textbf{0.8659} \\
\midrule
 \multirow{4}{*}{0.1}
 & \texttt{LoRA}
 & \underline{0.9370} $\pm$ 0.0058
 & \textbf{0.8098} $\pm$ 0.0076
 & \underline{0.8779} $\pm$ 0.0159
 & \underline{0.7253} $\pm$ 0.0440
 & \underline{0.8375} \\
 & \texttt{FFA-LoRA}
 & 0.9313 $\pm$ 0.0046
 & 0.7915 $\pm$ 0.0086
 & 0.8710 $\pm$ 0.0048
 & 0.7086 $\pm$ 0.0242
 & 0.8256 \\
 & \texttt{RoLoRA}
 & 0.9325 $\pm$ 0.0048
 & 0.7890 $\pm$ 0.0027
 & 0.8711 $\pm$ 0.0045
 & 0.7113 $\pm$ 0.0133
 & 0.8260 \\
 & \texttt{TAD-LoRA} (Ours)
 & \textbf{0.9401} $\pm$ 0.0019
 & \underline{0.8050} $\pm$ 0.0073
 & \textbf{0.8815} $\pm$ 0.0145
 & \textbf{0.7405} $\pm$ 0.0376
 & \textbf{0.8418} \\
\midrule
\multirow{4}{*}{0.05}
 &\texttt{ LoRA}
 & \underline{0.9330} $\pm$ 0.0030
 & \underline{0.7923} $\pm$ 0.0133
 & \underline{0.8764} $\pm$ 0.0071
 & \underline{0.7102} $\pm$ 0.0088
 & \underline{0.8280} \\
 & \texttt{FFA-LoRA}
 & 0.9263 $\pm$ 0.0043
 & 0.7814 $\pm$ 0.0145
 & 0.8659 $\pm$ 0.0147
 & 0.6699 $\pm$ 0.0079
 & 0.8109 \\
 & \texttt{RoLoRA}
 & 0.9238 $\pm$ 0.0036
 & 0.7851 $\pm$ 0.0089
 & 0.8424 $\pm$ 0.0129
 & 0.6818 $\pm$ 0.0130
 & 0.8083 \\
 & \texttt{TAD-LoRA} (Ours)
 & \textbf{0.9328} $\pm$ 0.0033
 & \textbf{0.7965} $\pm$ 0.0118
 & \textbf{0.8743} $\pm$ 0.0130
 & \textbf{0.7333} $\pm$ 0.0126
 & \textbf{0.8342} \\
\midrule
\multirow{4}{*}{0.02}
 & \texttt{LoRA}
 & 0.8668 $\pm$ 0.0811
 & \underline{0.7702} $\pm$ 0.0066
 & \underline{0.8416} $\pm$ 0.0031
 & \underline{0.6407} $\pm$ 0.0118
 & 0.7798 \\
 &\texttt{ FFA-LoRA}
 & \underline{0.9186} $\pm$ 0.0038
 & 0.7627 $\pm$ 0.0121
 & 0.8264 $\pm$ 0.0178
 & 0.6191 $\pm$ 0.0358
 & \underline{0.7817} \\
 & \texttt{RoLoRA}
 & 0.9147 $\pm$ 0.0022
 & 0.7582 $\pm$ 0.0116
 & 0.8160 $\pm$ 0.0045
 & 0.5980 $\pm$ 0.0291
 & 0.7717 \\
 & \texttt{TAD-LoRA} (ours)
 & \textbf{0.9263} $\pm$ 0.0019
 & \textbf{0.7783} $\pm$ 0.0036
 & \textbf{0.8480} $\pm$ 0.0063
 & \textbf{0.6604} $\pm$ 0.0408
 & \textbf{0.8032} \\

\midrule
\multirow{4}{*}{0.01}
 & \texttt{LoRA}
 & \underline{0.9133} $\pm$ 0.0044
 & 0.7503 $\pm$ 0.0081
 & \underline{0.8231} $\pm$ 0.0123
 & \underline{0.5862} $\pm$ 0.0396
 & \underline{0.7682} \\
 & \texttt{FFA-LoRA}
 & 0.9123 $\pm$ 0.0040
 & \underline{0.7555} $\pm$ 0.0132
 & 0.8178 $\pm$ 0.0080
 & 0.5759 $\pm$ 0.0242
 & 0.7654 \\
 & \texttt{RoLoRA}
 & 0.9120 $\pm$ 0.0014
 & 0.7420 $\pm$ 0.0209
 & 0.7974 $\pm$ 0.0143
 & 0.5153 $\pm$ 0.0110
 & 0.7417 \\
 & \texttt{TAD-LoRA} (Ours)
 & \textbf{0.9190} $\pm$ 0.0024
 & \textbf{0.7605} $\pm$ 0.0139
 & \textbf{0.8218} $\pm$ 0.0087
 & \textbf{0.6092} $\pm$ 0.0178
 & \textbf{0.7776} \\

\bottomrule
\end{tabular}
}
\end{table*}

\subsection{Weak-Regime Performance Summary}
\label{app:weak_regime}

We define the weak communication regime as $p \le 0.05$.
Table~\ref{tab:weak_regime_avg} reports the average performance
obtained by uniformly averaging the per-$p$ average accuracy across datasets.
This summary complements the main results by highlighting the
robust advantage of TAD-LoRA under communication-constrained settings.

\begin{table}[t]
\centering
\caption{
Average performance in the weak communication regime ($p \le 0.05$).
}
\label{tab:weak_regime_avg}
\begin{tabular}{l|c}
\toprule
Method & Weak-Regime Avg. Accuracy \\
\midrule
LoRA              & 0.7920 \\
FFA-LoRA         & 0.7860 \\
RoLoRA           & 0.7739 \\
TAD-LoRA (Ours)  & \textbf{0.8050} \\
\bottomrule
\end{tabular}
\end{table}

\subsection{Illustrative non-monotonic behavior.}
\begin{figure}[t]
  \centering
  \includegraphics[width=0.5\linewidth]{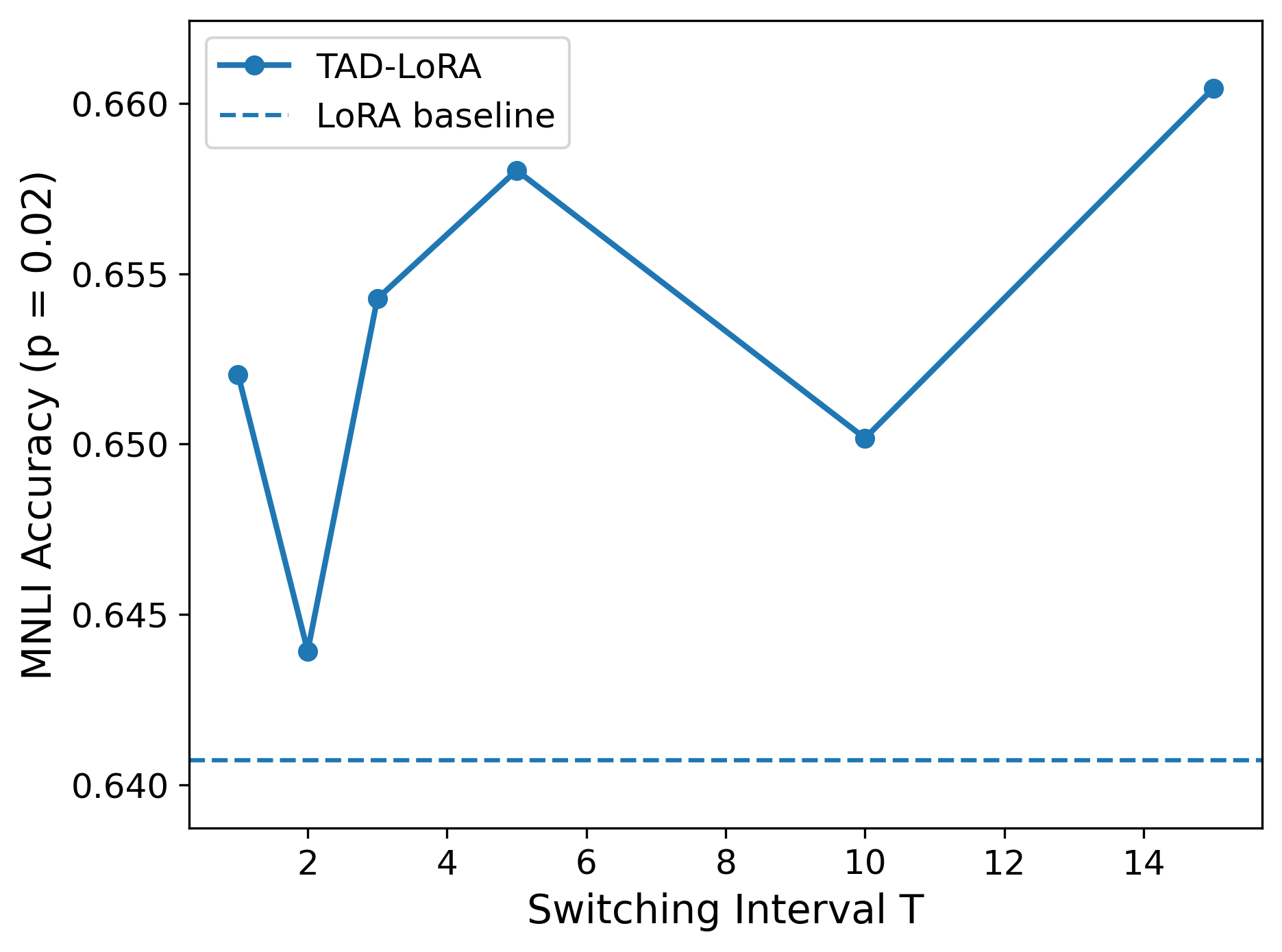}
  \caption{
  Illustrative example of the non-monotonic dependence on the switching interval $T$
  on MNLI under a weak communication regime ($p=0.02$).
  The y-axis reports the accuracy gain of TAD-LoRA over the LoRA baseline.
  Due to discrete scheduling and training noise,
  the U-shaped behavior predicted by theory manifests as a noisy and flattened trend
  rather than a sharp optimum.
  }
  \label{fig:u_shape_mnli}
\end{figure}

Figure~\ref{fig:u_shape_mnli} provides a representative example of the
non-monotonic dependence on the switching interval predicted by our theory.
In practice, the idealized U-shaped behavior is obscured by training noise,
discrete alternation schedules, and finite training budgets,
resulting in a range of effective switching intervals rather than a single
sharp optimum.
This observation is consistent with the aggregate trends and heatmap
visualizations reported in the main paper.
\begin{table}[t]
\centering
\caption{Empirically selected optimal switching interval $\hat T^*(p)$ under different communication probabilities $p$.
The candidate set is consistent with experimental setting.
}
\label{tab:bestT_raw}
\begin{tabular}{c|cccc|cc}
\toprule
$p$ & SST-2 & QQP & QNLI & MNLI & Avg. & Median \\
\midrule
0.5  & 1  & 1  & 1  & 3  & 1.50 & 1 \\
0.2  & 1  & 5  & 1  & 5  & 3.00 & 3 \\
0.1  & 5  & 3  & 3  & 3  & 3.50 & 3 \\
0.05 & 5  & 5  & 15 & 5  & 7.50 & 5 \\
0.02 & 3  & 1  & 15 & 15 & 8.50 & 9 \\
0.01 & 10 & 5  & 3  & 3  & 5.25 & 4 \\
\bottomrule
\end{tabular}
\end{table}
\subsection{Additional Network Topologies}

\subsubsection{Ring Topology}
Beyond the representative random (Erd\H{o}s--R\'enyi) communication topology
considered in the main results, we further evaluate \texttt{TAD-LoRA} under a ring
topology, which represents a highly structured network with limited connectivity
and slow information mixing.
This setting serves as an extreme stress test rather than a typical deployment
scenario.

As shown in Table~\ref{tab:appendix_ring_results}, all methods experience reduced
performance compared to random topologies, which is expected due to the
restricted information propagation in a ring.
Nevertheless, \texttt{TAD-LoRA} remains competitive with baseline methods and
achieves comparable average accuracy with stable variance.
In particular, \texttt{TAD-LoRA} attains the best performance on the more
challenging MNLI task, while consistently outperforming other alternating
baselines.

Overall, these results indicate that while the advantage of topology-aware
coordination is diminished when the network structure itself becomes the
dominant bottleneck, \texttt{TAD-LoRA} does not introduce additional instability
or degradation under extremely constrained topologies.

\begin{table*}[t]
\centering
\caption{Results under a ring communication topology.
Test accuracy is reported as mean $\pm$ variance over 3 random seeds.
The ring topology represents a structured network with slow information mixing
and serves as a stress test for topology-aware methods.}
\label{tab:appendix_ring_results}
\resizebox{\textwidth}{!}{
\begin{tabular}{l|cccc|c}
\toprule
Method & SST-2 & QQP & QNLI & MNLI & Avg. \\
\midrule
LoRA
& 0.9464 $\pm$ 0.0016
& 0.8308 $\pm$ 0.0057
& 0.9015 $\pm$ 0.0048
& 0.8185 $\pm$ 0.0018
& 0.8743 \\
FFA-LoRA
& 0.9392 $\pm$ 0.0026
& 0.8044 $\pm$ 0.0096
& 0.8977 $\pm$ 0.0074
& 0.7728 $\pm$ 0.0150
& 0.8535 \\
RoLoRA
& 0.9419 $\pm$ 0.0002
& 0.8246 $\pm$ 0.0024
& 0.8957 $\pm$ 0.0064
& 0.7997 $\pm$ 0.0112
& 0.8655 \\
TAD-LoRA
& 0.9438 $\pm$ 0.0012
& 0.8290 $\pm$ 0.0046
& 0.9010 $\pm$ 0.0020
& 0.8208 $\pm$ 0.0046
& 0.8736 \\
\bottomrule
\end{tabular}
}
\end{table*}

\end{document}